\documentclass{article} 
\usepackage{iclr2026_conference,times}

\usepackage{amsmath,amsfonts,bm}

\def\eqref#1{equation~\ref{#1}}

\def\1{\bm{1}}

\def\rvb{{\mathbf{b}}}

\def\rvr{{\mathbf{r}}}

\def\rvw{{\mathbf{w}}}
\def\rvx{{\mathbf{x}}}
\def\rvy{{\mathbf{y}}}

\def\mI{{\bm{I}}}

\DeclareMathAlphabet{\mathsfit}{\encodingdefault}{\sfdefault}{m}{sl}
\SetMathAlphabet{\mathsfit}{bold}{\encodingdefault}{\sfdefault}{bx}{n}

\def\gL{{\mathcal{L}}}

\def\gU{{\mathcal{U}}}

\def\gX{{\mathcal{X}}}

\def\sP{{\mathbb{P}}}

\def\sR{{\mathbb{R}}}

\newcommand{\R}{\mathbb{R}}

\newcommand{\meanp}[2]{\mathbb{E}_{#1} \left\lbrack #2 \right\rbrack}

\newcommand{\norm}[1]{\lVert#1\rVert}
\newcommand{\normscaled}[1]{\left\lVert#1\right\rVert}

\def\bridge{\rvb\rvr}

\usepackage{hyperref}
\usepackage{url}
\usepackage[capitalize,noabbrev]{cleveref}

\usepackage{microtype}
\usepackage{graphicx}
\usepackage{subcaption}
\usepackage{booktabs} 

\usepackage{amsmath}
\usepackage{amssymb}
\usepackage{mathtools}
\usepackage{amsthm}

\usepackage[normalem]{ulem}

\title{Improved Constrained Generation by Bridging Pretrained Generative Models}

\author{Xiaoxuan Liang$^{1,2}$, Saeid Naderiparizi$^{1,2}$, Yunpeng Liu$^{1,2}$, 
\\ \textbf{Berend Zwartsenberg}$^{2}$, \textbf{Frank Wood}$^{1,2}$\\
$^{1}$University of British Columbia, $^{2}$Inverted AI\\
\texttt{liang51@cs.ubc.ca} \\
}

\theoremstyle{plain}
\newtheorem{theorem}{Theorem}[section]

\theoremstyle{definition}

\theoremstyle{remark}

\iclrfinalcopy 
\def\method{MBM++}
\begin{document}

\maketitle

\begin{abstract}
  Constrained generative modeling is fundamental to applications such as robotic control  and autonomous driving, where models must respect physical laws and safety-critical constraints. In real-world settings, these constraints rarely take the form of simple linear inequalities, but instead complex feasible regions that resemble road maps or other structured spatial domains. We propose a constrained generation framework that generates samples directly within such feasible regions while preserving realism. Our method fine-tunes a pretrained generative model to enforce constraints while maintaining generative fidelity. Experimentally, our method exhibits characteristics distinct from existing fine-tuning and training-free constrained baselines, revealing a new compromise between constraint satisfaction and sampling quality.
\end{abstract}

\section{Introduction}
Diffusion models~\citep{song2019generative, ho2020denoising, songscore, karras2022elucidating} and flow matching~\citep{liu2022flow, albergo2022building, lipman2022flow} have emerged as powerful generative frameworks, achieving strong performance across a wide range of high-dimensional generation tasks~\citep{rombach2022high}. Their success arises from learning a data-driven generative process that accurately captures complex distributions through denoising or velocity field prediction. However, when deployed in safety-critical or physics-constrained domains, such as robotic control~\citep{schulman2014motion, carvalho2023motion} and autonomous driving~\citep{niedoba2024diffusion, yang2024diffusion, zheng2025diffusion}, samples produced by pretrained models often violate task-specific constraints, including collision avoidance and drivable area compliance. Addressing these violations is essential for deploying generative models in real-world systems.

A key challenge in constrained generation is how to incorporate constraint information into the generative process without disrupting the learned data distribution. Many constraints encountered in practical domains are highly nonlinear and state-dependent, and are often specified implicitly via loss functions rather than explicit, closed-form feasible sets. In contrast, some recent constrained generation settings assume simple linear or explicitly parameterized constraints. For example, some works assume linear watermarking constraints~\citep{liu2023mirror} or analytically defined feasibility regions in multi-agent planning~\citep{liang2025simultaneous}, which admit direct projection or closed-form enforcement. Such assumptions, however, rarely hold in realistic scenarios where feasibility depends on complex geometry, interactions, or learned environment representations. Bridging this gap requires constraint signals that are defined in data space, expressive enough to capture complex feasibility criteria and simultaneously compatible with the dynamics of diffusion and flow-matching models.

Motivated by these challenges, we propose a principled framework for constrained generative modeling that integrates implicit, loss-based constraints directly into the training dynamics of pretrained diffusion and flow-matching models. Rather than enforcing feasibility through explicit projection, our approach instead adapts the generative process to better align generated samples with constraint satisfaction.

Our contributions are as follows:
\begin{enumerate}
    \item We propose \method{}, a fine-tuning framework for constrained generative modeling. The method evaluates constraint losses on the one-step denoised estimate and uses the resulting guided vector to guide sampling. The guided vector is optimized jointly with the original denoising score matching or flow matching objective. Constraint information is injected via a lightweight MLP-based embedding, while the pretrained backbone remains fixed. The design enables stable constraint enforcement without explicit manifold projection. The framework applies to both diffusion and flow-matching models. We refer to this module as a bridge embedding.
    
    \item We empirically show that ~\method{} exhibits behavior distinct from both training-free guidance and fine-tuning baselines, revealing a different compromise between constraint satisfaction and sampling quality.
\end{enumerate}

\begin{figure*}[t]
    \centering
    \begin{subfigure}[b]{0.32\textwidth}
        \centering
       \includegraphics[width=1\textwidth]{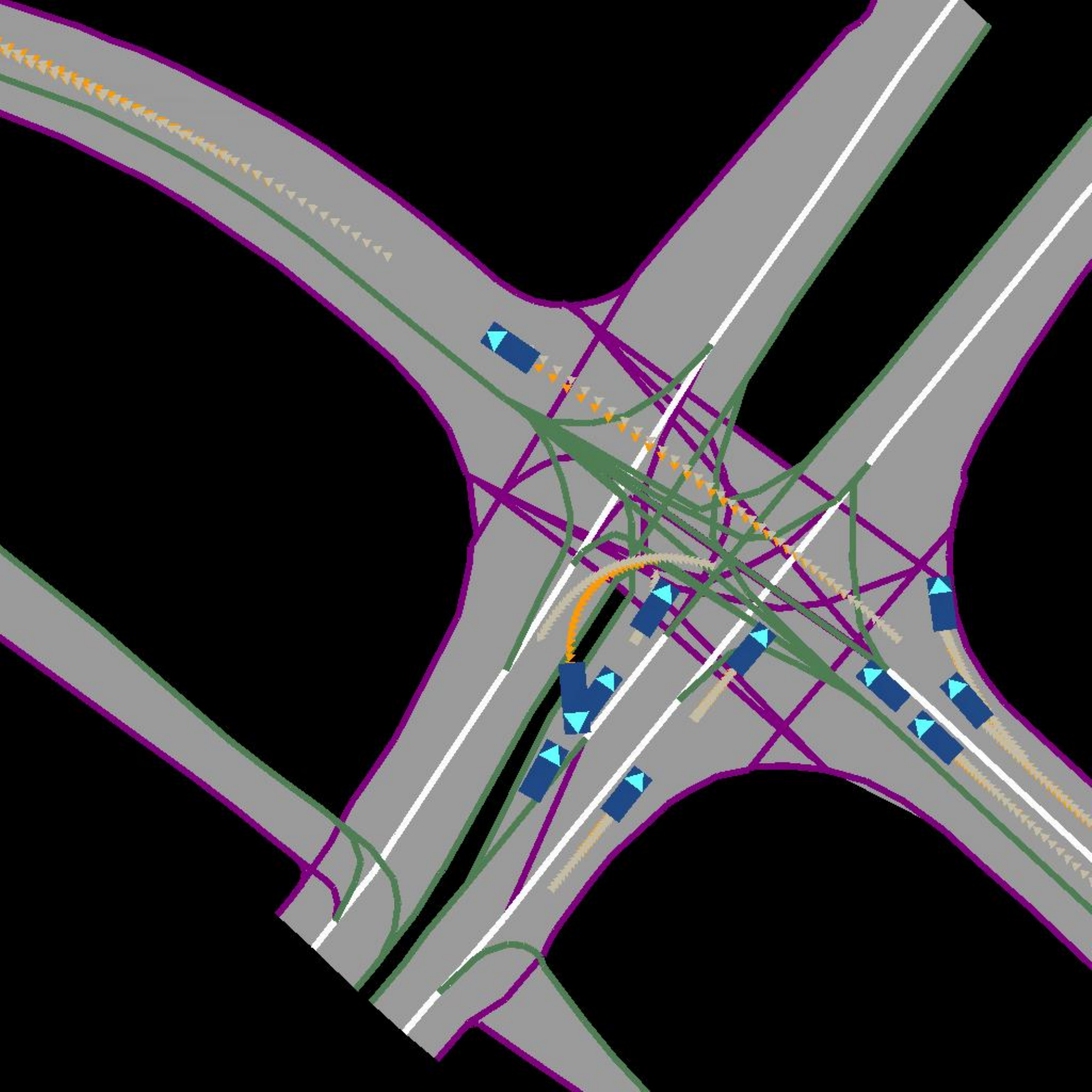}
       \caption{Standard Diffusion}
        \label{fig:banner-a}
    \end{subfigure}
    \hfill
    \begin{subfigure}[b]{0.32\textwidth}
        \centering
        \includegraphics[width=1\textwidth]{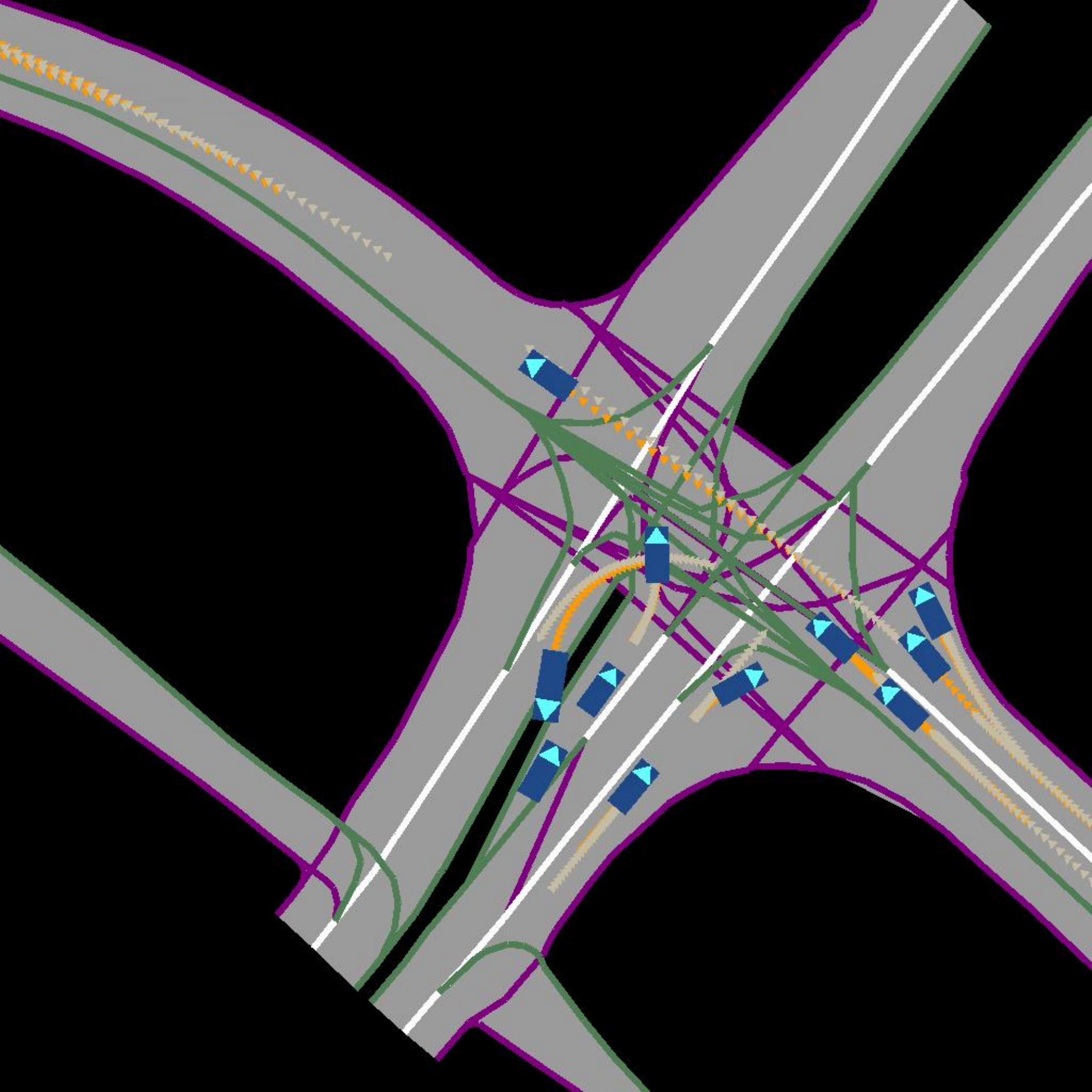}
        \caption{Training-free Guidance}
        \label{fig:banner-b}
    \end{subfigure}
    \hfill
    \begin{subfigure}[b]{0.32\textwidth}
        \centering
        \includegraphics[width=1\textwidth]{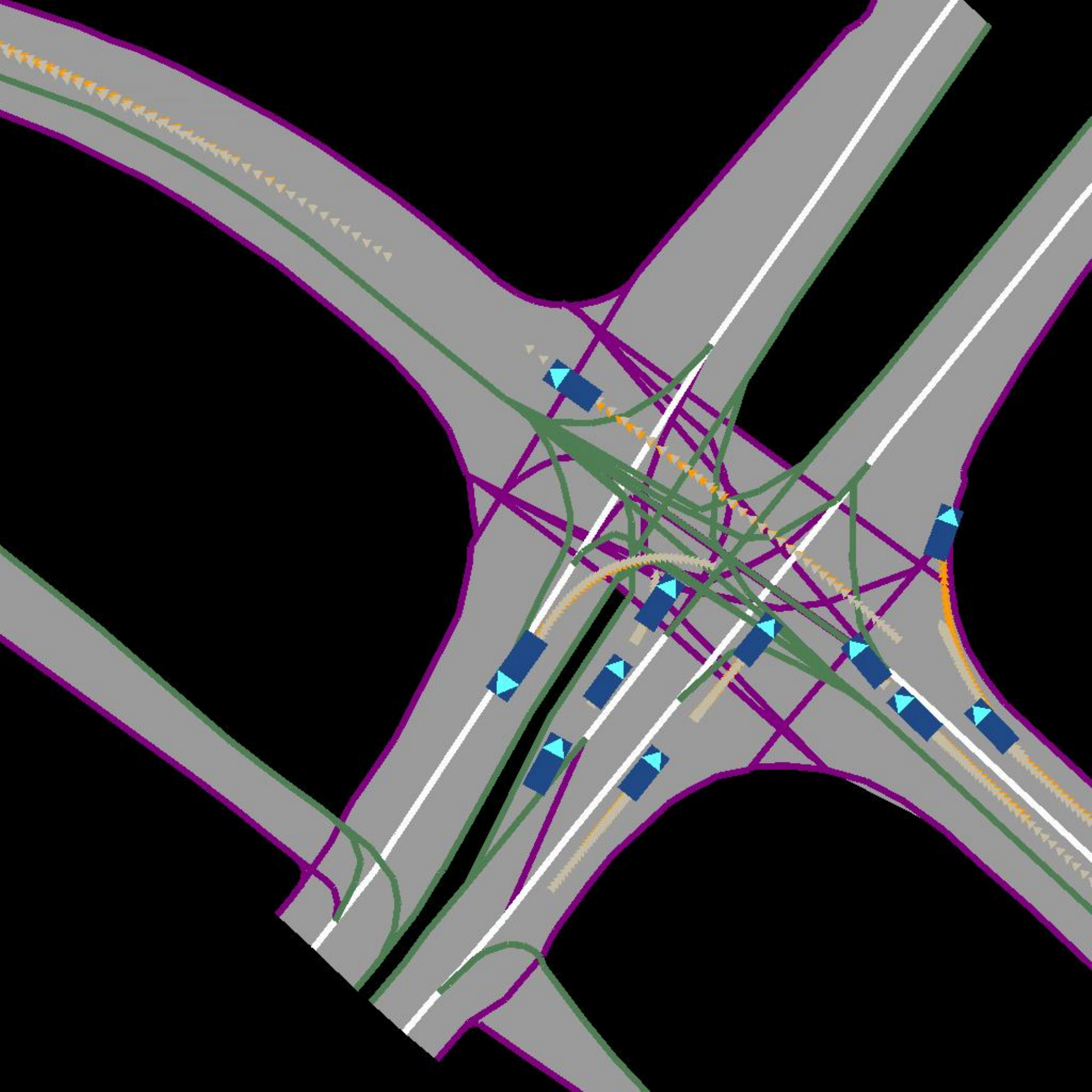}
        \caption{\method{} (Ours)}
        \label{fig:banner-c}
    \end{subfigure}
    \caption{Qualitative comparison of constrained trajectory generation. We visualize samples from the same traffic intersection scene and timestep, where gray dot trajectories indicate ground-truth agent motion and orange trajectories denote model predictions. Panel~\subref{fig:banner-a} shows an unconstrained diffusion baseline, which exhibits both offroad and collision violations while executing a left turn. Panel~\subref{fig:banner-b} shows a training-free guided method, MPGD without projection~\citep{hemanifold} that reduces violations but introduces noticeable trajectory distortion, including persistent offroad behavior and an abrupt acceleration during a right-turn maneuver toward the northeast, leading to a near-collision. Panel~\subref{fig:banner-c} shows our \method{}, which exhibits no offroad or collision violations while preserving realistic and coherent motion. All samples are generated from the same initial conditions. }
    \label{fig:banner}  
\end{figure*}

\section{Background}
\paragraph{Problem Setup}
\label{sec: definition}
A constrained generation problem is defined by an \textit{ambient sample space}, a \textit{constraint set}, and a \textit{target distribution supported on constraint set}. Let $\gX$ denote the ambient sample space, and $\Omega \subseteq \gX$ denote the constraint set.
A sample $\rvx \in \gX$ is ``valid'' only if $\rvx \in \Omega$, meaning it satisfies the problem constraints.
Let $p_0(\rvx)$ denote an unknown data generating distribution supported on $\gX$, from which we observe a collection of i.i.d. samples as a training dataset. In particular, $p_0$ satisfies$\int_{\rvx \in \Omega} p_0(\rvx)\,d\rvx = 1$. Our goal is to learn a generative model $p_\theta$ defined over the ambient sample space $\gX$, such that its induced distribution concentrates on $\Omega$ and matches the data distribution within the constraint set. We assume that $\Omega$ is implicitly specified through a non-negative differentiable loss function $\ell^\Omega: \gX \rightarrow \R^{\geq 0}$ such that $\ell^\Omega(\rvx) = 0$ if and only if $\rvx \in \Omega$.


\subsection{Diffusion Models}
Diffusion models (DMs)~\citep{sohl2015deep,ho2020denoising,songscore} are a class of generative models that learn to reverse a stochastic process to generate samples from a complex data distribution. They gradually transform data to pure noise through a \emph{forward process} defined by a Stochastic Differential Equation (SDE)
\begin{equation}
    \sP: d\rvx_t = f(\rvx_t; t) dt + \sigma(t) d\rvw, \quad \rvx_0 \sim p_0,
    \label{eq:diffusion:forward}
\end{equation}
where $\rvw$ is a standard Wiener process, and $f$ and $\sigma$ are predefined drift and diffusion functions. The process is constructed so that, for a sufficiently large $T$, the marginal distribution $p_T$ approaches a simple prior $\pi$, typically a standard Gaussian. The time-reversed process of \cref{eq:diffusion:forward} follows another SDE known as the \textit{reverse-time process}~\citep{anderson1982reverse},
\begin{equation}
    d\rvx_t = [f(\rvx_t; t) - \sigma^2(t) \nabla \log p_t(\rvx_t)]\,dt + \sigma(t)\,d\bar{\rvw},
    \label{eq:diffusion:reverse}
\end{equation}
where $\rvx_T \sim \pi$ and $\bar{\rvw}$ is a reverse-time Wiener process. Since the score function $\nabla_{\rvx_t} \log p_t(\rvx_t)$ is generally intractable, diffusion models approximate it using a neural network $s_\theta(\rvx_t, t)$, trained via minimizing the denoising score matching (DSM) objective~\citep{vincent2011connection}
\begin{equation}
\label{eq: score-learning}
    \gL(\theta) = \meanp{t, \rvx_0, \rvx_t}{\lambda(t)\norm{s_\theta(\rvx_t, t) - \nabla_{\rvx_t} \log p_t(\rvx_t \mid \rvx_0)}^2},
\end{equation}
which is an expectation over $t \sim U(t_\mathrm{min}, T)$ for some positive $t_\mathrm{min} \approx 0$, $\rvx_0$ is drawn from the data distribution, and $\rvx_t \sim p(\rvx_t \mid \rvx_0)$.
Sampling is performed by numerically simulating the learned reverse-time SDE starting from $\rvx_T\sim \pi(\cdot)$:
\begin{equation}
    d\rvx_t = [f(\rvx_t; t) - \sigma^2(t) s_\theta(\rvx_t; t)]\,dt + \sigma(t)\,d\bar{\rvw}.
    \label{eq:diffusion:reverse-approx}
\end{equation}

Moreover, diffusion models can be viewed as learning time-dependent dynamics that map a simple prior $\pi$ to the data distribution.
For Gaussian perturbations, the marginal score satisfies
\begin{align}
    \label{eq: diffusion-param}
    \nabla_{\rvx_t} \log p_t(\rvx_t) = \frac{\meanp{}{\rvx_0 \mid \rvx_t} - \rvx_t}{\sigma^2(t)},
\end{align}
where the posterior mean is defined as $\meanp{}{\rvx_0\mid \rvx_t} = \int \rvx_0 p(\rvx_0\mid \rvx_t)\,d\rvx_0$. A derivation is provided in Appendix~\cref{subsec: posterior-mean}. This identity motivates a parameterization: rather than directly learning the score, we train a neural network parameterized by $\theta$ to predict $\meanp{}{\rvx_0\mid \rvx_t}$, equivalently, a denoised estimate of the clean sample $D_\theta(\rvx_t; t)$. The score is then obtained by
\begin{align}
    s_\theta(\rvx_t; t) := \frac{D_\theta(\rvx_t; t) - \rvx_t}{\sigma^2(t)}.
\end{align}
This conditional posterior mean parameterization will later allow us to draw connections to deterministic generative models. A detailed discussion of flow matching is deferred to Appendix~\cref{sec: flow-matching}.

\subsection{Manually Bridged Diffusion Models}
\citet{naderiparizi2025constrained} present a framework called Manually Bridged Model (MBM), a constraint-aware extension of score-based diffusion models~\citep{songscore}. Given a constraint loss function $\ell^\Omega(\cdot)$ from~\cref{sec: definition}, a manual bridge term $\bridge_{\mathrm{MBM}}^\Omega(\rvx_t; t)$ is defined as $\gamma(t)\nabla_{\rvx_t} \ell^\Omega(\rvx_t)$, where
$\gamma(t): (0, T]\to \R^+$ is a  $C^1$ function, 
\begin{align}
    \label{eq: gamma-property}
    \lim_{t\to T}\gamma(t) = 0, \quad \lim_{t\to 0} \gamma(t) \to \infty,
\end{align} which vanishes at high noise levels and diverges at $t\to 0$, so that constraints are effectively enforced near the low-noise data manifold. The score model is modified as
\begin{align}
    s_\theta^\Omega(\rvx_t; t, \ell^\Omega, \gamma) := s_\theta(\rvx_t; t) - \bridge_{\mathrm{MBM}}^\Omega(\rvx_t; t),
\end{align}
and is trained via the standard DSM objective~\cref{eq: score-learning}.
Under the EDM~\citep{karras2022elucidating} variance-exploding parameterization, the resulting reverse SDE is:
\begin{align}
\label{eq: MBM SDE}
    d\rvx_t &= -2\sigma(t)\,s_\theta^\Omega(\rvx_t; t, \ell^\Omega, \gamma)\,dt + \sqrt{2\sigma(t)}\,d\bar\rvw,
\end{align}
which incorporates constraints both implicitly via score-model fine-tuning and explicitly
via an additive gradient field correction to  the model output. The corresponding marginal density satisfies:
\begin{align}
    p_{\mathrm{MBM}}^\Omega(\rvx_t; t)\propto p(\rvx_t; t)\exp(-\gamma(t) \ell^\Omega(\rvx_t; t))
\end{align}
such that 
\begin{align}
    p_{\mathrm{MBM}}^\Omega(\rvx_T; T) &= p(\rvx_T; T) \nonumber \\
    \lim_{t\to 0}p_{\mathrm{MBM}}^\Omega(\rvx_t; t)&= \lim_{t\to 0} p(\rvx_t; t)\1_\Omega(\rvx_t).
\end{align} 
This implies at the terminal state, the distribution concentrates on the feasible set, converging to the data distribution restricted to $\Omega$.
This construction extends naturally to multiple independent constraints by summing their bridge terms.

\section{Methodology}
We present~\method{}, a guidance-embedded fine-tuning framework for constrained generative modeling. Building on MBM, which evaluates constraint gradients directly at noisy states, \method{} instead evaluates constraints on the one-step denoised estimate, effectively shifting guidance from noisy space to data space and yielding gradients better aligned with semantic constraint violations. This modification yields more stable and informative guidance, especially at high noise levels. \method{} integrates denoised-state guidance directly into the learned score/vector field during training and applies uniformly to both diffusion and flow matching models.
\subsection{Denoised State Constraint Guidance During Training}
We consider a noisy state $\rvx_t$ along the generative trajectory and the choice of where to evaluate a constraint loss $\ell^\Omega(\cdot)$ in order to compute guidance gradients. Given a noisy intermediate state $\rvx_t$, the loss can be evaluated either directly at the noisy state $\rvx_t$ or at its corresponding one-step denoised estimate $D_\theta(\rvx_t; t)$ predicted by the pretrained base model parameterized by $\theta$. Prior methods such as MBM evaluate constraint gradients at $\rvx_t$, which can lead to high-variance gradient estimates at high noise levels, where $\rvx_t$ lies far from the noiseless data manifold. Consequently, gradients are computed far outside the feasible region of the constraints, where they are highly sensitive to noise and provide unreliable guidance. A closely related approach is adjoint matching~\citep{domingo2024adjoint}, which avoids noisy-state loss evaluation by formulating guidance through the minimization of a cost functional under a stochastic optimal control~\citep{fleming2012deterministic, sethi2018optimal}. In this framework, reverse-time trajectories are explicitly simulated to the terminal state, where constraint costs are evaluated, and guidance gradients are obtained by integrating the corresponding adjoint dynamics backward through time. While yielding accurate guidance signals, it requires full trajectory rollout and adjoint integration during training, incurring substantially higher computation and memory costs. We discuss its relationship to our method in the Appendix~\cref{sec: connection-to-soc}. Instead in \method{}, constraint guidance is computed by evaluating the constraint loss on the denoised estimate, which is closer to the underlying data distribution. Concretely, we define the bridge term (abusing notation) $\bridge_t^\Omega(\rvx_t; t)$ as:
\begin{align}
    \bridge_t^\Omega(\rvx_t; t) = \gamma(t)\nabla_{\rvx} \ell^\Omega(\rvx)\Big|_{\rvx = \mathrm{sg}(D_\theta(\rvx_t; t))},
\end{align}
where the gradient is taken with respect to the denoised state and $\gamma(t)$ controls the strength of constraint enforcement over time, which follows~\cref{eq: gamma-property} . The stop-gradient operator $\mathrm{sg}(\cdot)$ indicates that the constraint gradient is evaluated on the denoised estimate but is not backpropagated through the denoiser itself, avoiding differentiation through the pretrained network and substantially reducing computational and memory overhead.


\subsection{Bridge Conditioning for Fine-tuning}
\label{sec: bridge-finetune}
We now describe how the guidance signal $\bridge^\Omega(\rvx_t; t)$ is incorporated into a pretrained model. We focus on diffusion models in the EDM framework~\citep{karras2022elucidating}, where the pretrained model is implemented as the denoiser $D_\theta(\cdot; t)$, which takes a noisy state $\rvx_t$ as input and predicts a corresponding clean data estimate. At each training step, we first compute the denoised estimate $D_\theta(\rvx_t; t)$ using a forward pass of the frozen pretrained model without any constraint conditioning, and apply the stop-gradient operator, yielding $\mathrm{sg}(D_\theta(\rvx_t; t))$, when evaluating the constraint gradient. In contrast to MBM, which fine-tunes the full generative model, \method{} introduces a lightweight MLP-based bridge embedding module parameterized by $\phi$. The entire pretrained architecture is reused and kept frozen during fine-tuning, with the bridge embedding as the sole additional trainable component. This design preserves generative coverage and improves optimization stability.

 The constraint information is injected at both the input and output of the fine-tuning model, which augments a frozen pretrained backbone with a lightweight trainable embedding. In the pretrained model, the model input is the learned embedding $E_\theta(\rvx_t; t)$ of the noisy state $\rvx_t$. During fine-tuning, all backbone parameters $\theta$, including $E_\theta$, are frozen. To inject constraint information without modifying the backbone, we introduce a separate learnable guidance embedding $E_\phi$, parameterized by $\phi$. This embedding encodes the constraint gradient and is added to the original model input. The resulting input to the model is
\begin{equation}
\label{eq: input-inj}
    \phi(\rvx_t; t) = E_\theta(\rvx_t; t) + E_\phi\left(\nabla_\rvx \ell^\Omega(\rvx) \Big |_{\rvx = \mathrm{sg}(D_\theta(\rvx_t; t))}\right).
\end{equation}
 The input-side embedding allows the frozen backbone of the fine-tuning model to adapt its internal representations to the constraint signal. On the output side, a residual correction derived from the same bridge signal is added to the model output, 
\begin{align}
\label{eq: output-inj}
    D^{\Omega}_{\theta, \phi}(\rvx_t; t) = D_\theta(\phi(\rvx_t; t); t) - \bridge^\Omega(\rvx_t; t).
\end{align}
The output side correction directly counteracts constraint violations in the predicted clean state.

During fine-tuning, gradients are propagated only to $\phi$ with the constraint-augmented objective function given the parameterization~\cref{eq: diffusion-param}:
\begin{align}
    \label{obj: finetune-obj}
    \gL^{\mathrm{DM}}(\phi) &= \meanp{}{\lambda(t)\normscaled{ s_{\theta, \phi}^\Omega - \nabla \log p_t(\rvx_t \mid \rvx_0)}^2}, \quad\mathrm{s.t.}\ &s^{\Omega}_{\theta, \phi}(\rvx_t; t) = \frac{D_{\theta, \phi}^\Omega(\rvx_t; t) - \rvx_t}{\sigma^2(t)},
\end{align}
since the constraint-conditioned model output $D_{\theta, \phi}^\Omega(\rvx_t; t)$ approximates a clean data sample, the constraint-conditioned model is trained using the standard DSM objective. This design preserves the original training formulation by freezing the pretrained backbone and fine-tuning only a lightweight bridge embedding module. Constraint information, evaluated on the denoised estimate, is incorporated through a lightweight bridge embedding at the input (\cref{eq: input-inj}), and a residual correction at the output (\cref{eq: output-inj}), enabling effective constraint enforcement without altering the pretrained generative model and resulting in a stable and parameter-efficient fine-tuning procedure.

\paragraph{Relation to Sampling-time Denoised Guidance} 
Prior methods such as MPGD~\citep{hemanifold} evaluate constraint losses on the denoised estimate $D_\theta(\rvx_t; t)$ only at sampling time and use the resulting gradient scaled by some weighting function $r(t)$ to modify the reverse-time dynamics of a fixed pretrained model:
\begin{align}
    s_{\mathrm{MPGD}}(\rvx_t; t) = \frac{D_\theta(\rvx_t; t) - r(t) \nabla_\rvx\ell^\Omega(\rvx)\Big |_{\rvx=\mathrm{sg}(D_\theta(\rvx_t; t))} - \rvx_t}{\sigma^2(t)}.
\end{align}
Because this modification is applied post hoc, MPGD additionally performs an explicit projection step that uses a pretrained decoder to map the guided sample back onto the data manifold, preserving sample quality. In contrast, \method{} incorporates denoised-state constraint signals during training by optimizing the DSM objective with a constraint-conditioned score $s_{\theta, \phi}^\Omega(\rvx_t; t)$. 

\subsection{Theoretical Analysis for \method{}}
\begin{theorem}(\textbf{Asymptotic validity of constraint gradient substitution}) 
    \label{theorem: theorem}
Assume:
    \begin{itemize}
        \item Forward kernel concentrates: for the diffusion process, $\rvx_t\xrightarrow{p} \rvx_0$ as $t\downarrow 0$ and $p(\rvx_t\mid \rvx_0) \to \delta(\rvx_t - \rvx_0)$.
        \item Denoising consistency: the denoised estimate $D_\theta(\rvx_t; t)$ is consistent at small $t$: $D_\theta(\rvx_t; t) - \rvx_t \xrightarrow{p} 0$ as $t\downarrow 0$.
        \item Smoothness: loss function $\ell^\Omega(\cdot)$ has $L$-Lipschitz gradient: $\norm{\nabla \ell(\rvx) - \nabla \ell(\rvy)} \leq L\norm{\rvx - \rvy}$.
        \item Identity Jacobian consistency: the Jacobian of denoised estimate $D_\theta(\rvx_t; t)$ with respect to $\rvx_t$ approaches identity at small $t$: $\normscaled{\frac{\partial D_\theta(\rvx_t; t)}{\partial \rvx_t} - \mI}\xrightarrow{p} 0$  as $t\downarrow 0$.
    \end{itemize}
    Then
    \begin{align}
        \lim_{t\downarrow 0}\norm{\nabla_{\rvx_t} \ell^\Omega (\rvx_t) - \nabla_\rvx\ell^\Omega(\rvx)\Big|_{\rvx = \mathrm{sg}(D_\theta(\rvx_t; t))}}\xrightarrow{p} 0.
    \end{align}
\end{theorem}

See Appendix~\cref{sec: proof} for proof details.

\section{Experiments}
We evaluate our method on both a dynamical system describing the motion of colliding balls in a box and real-world traffic scene generation. The bouncing ball experiment isolates physical and boundary constraints, capturing collision dynamics between balls as well as contacts with the enclosing walls. The traffic scene experiment extends evaluation to safety-critical scenarios with complex geometric and interaction constraints. Together, these experiments allow us to assess whether constraint-aware fine-tuning can reduce violations while preserving the learned generative distribution across both synthetic and realistic settings.

\subsection{Bouncing Balls}
We demonstrate the effectiveness of \method{} on a time-series modeling task involving the prediction of trajectories for a collection of balls interacting in a closed box. In this task, the balls move in straight paths at constant velocities until they collide with each other or the boundaries of the scene. Upon collision, the balls bounce elastically and deterministically, following the laws of momentum conservation.

The physics-based simulator used in this experiment is adapted from~\citet{gan2015deep}. We generated a training set consisting of 100,000 scenarios, each spanning 100 timesteps. At the start of each scenario, 10 balls are randomly positioned with random initial velocities, and their motion is simulated based on the physics model. The goal of this task is to learn the underlying simulator using the generated training dataset, and the task imposes constraints to prevent ball overlaps (termed  ``collision infractions'') and boundary violations (termed ``boundary infractions'').

We compare \method{} with a broad range of constrained generation methods across both diffusion and flow-matching frameworks. For diffusion models, we include standard EDM framework~\citep{karras2022elucidating} as an unconstrained baseline, MPGD without projection~\citep{hemanifold} as a training-free guidance method that evaluates the constraint loss at the one-step denoised estimate during sampling, and MBM~\citep{naderiparizi2025constrained}, which fine-tunes the model by incorporating gradients of constraint losses evaluated at every noisy intermediate state. We further consider standard FM as a deterministic alternative and adjoint matching (AM)~\citep{domingo2024adjoint}, a fine-tuning method which shares a similar control-theoretic motivation to our approach; details are provided in Appendix~\cref{sec: connection-to-soc}. To mirror MPGD in the FM setting, we implement a training-free MPGD style guidance for FM that evaluates the constraint loss at the one-step denoised estimate as well. Finally, we apply our proposed method to both diffusion and FM models for performance comparisons. 

Since each scenario is initialized with randomly sampled ball positions and velocities, there is no sample-wise correspondence between generated and ground-truth trajectories, and we therefore focus on distributional metrics. In addition to reporting ELBO, we evaluate distribution shift by computing the directed Hausdorff distance (termed ``HDH'') from generated states to the training data support by treating each state as point in $\sR^2$ and measuring the maximum nearest-neighbor distance from generated states to the ground truth states. This metric quantifies how far unconditional samples deviate from the training distribution.

In~\cref{tab:bb-experiment}, we observe a clear trade-off between constraint satisfaction and distributional fidelity. Unconstrained baselines preserve likelihood but suffer from high collision and boundary violation rates. The training-free guided methods, e.g. MPGD w/o projection and TF-Guided FM, nearly eliminate violations but induce degradation in ELBO and Hausdorff distance, indicating a distributional shift and degraded sample quality. Prior fine-tuning methods, e.g. MBM and AM, update the entire base model, altering the pretrained weights; in particular, AM incurs additional computational overhead, requiring a solution of a controlled SDE during training time (see~\cref{sec: connection-to-soc}), which leads to a degraded performance compared to our method. While MBM substantially reduces violations, it still exhibits non-negligible collision and boundary rates. Our ~\method{} further suppresses these infractions while maintaining r-ELBO comparable to MBM and achieving a lower HDH overall.

 \begin{table*}[!t] 
  \centering
  \small
  \caption{Bouncing balls experiment results. Best results are in \textbf{bold}. Second best results are $\uline{\mathrm{underlined}}$.}
  \begin{tabular}{llllll}
    \toprule
    \centering
    Method     & Collision ($\%$)  $\downarrow$  & Boundary ($\%$) $\downarrow$ & r-ELBO($\times 10^{-2}$) $\uparrow$  & HDH ($\times 10^{-2}$) $\downarrow$ \\
    \midrule
    Standard diffusion  & $38.12 \pm 0.36$ & $6.29 \pm 0.06$ & $\mathbf{-21.7 \pm 0.1}$ & $28.9 \pm 1.23$  \\
    MPGD w/o projection & $\mathbf{0.00 \pm 0.00}$ & $\mathbf{0.00 \pm 0.00}$ & $-27.1 \pm 0.1$ & $\uline{27.7 \pm 0.5}$ \\
    MBM & $2.86 \pm 0.04$ & $1.91 \pm 0.03$ & $\uline{-22.2 \pm 0.1}$ & $28.4 \pm 0.6$ \\
    DM-\method{}\ (Ours) & $\uline{0.01 \pm 0.00}$ & $\uline{0.03\pm 0.00}$ & $-22.8 \pm 0.1$ & $\mathbf{27.5 \pm 1.3}$\\
    \midrule
    Standard FM & $29.17 \pm 0.31$ & $3.80 \pm 0.08$ & $\mathbf{-37.3 \pm 0.1}$ & $\mathbf{28.5 \pm 2.0}$ \\
    TF-Guided FM & $\mathbf{1.03 \pm 0.03}$ & $\mathbf{0.19 \pm 0.02}$ & $-52.9 \pm 0.3$ & $30.4 \pm 1.4$  \\
    AM & $4.67 \pm 0.06$ & $0.49 \pm 0.02$ & $-39.4 \pm 0.2$ & $30.4 \pm 1.5$ \\
    FM-\method{}\ (Ours) & $\uline{2.35 \pm 0.08}$ & $\uline{0.44 \pm 0.03}$ & $\uline{-37.7 \pm 0.1}$ & $\uline{28.7 \pm 2.3}$ \\
    \bottomrule
  \end{tabular}
  \label{tab:bb-experiment}
\end{table*}

\begin{figure*}[!ht]
    \centering
    \begin{subfigure}[b]{0.32\textwidth}
        \centering
       \includegraphics[width=1\textwidth]{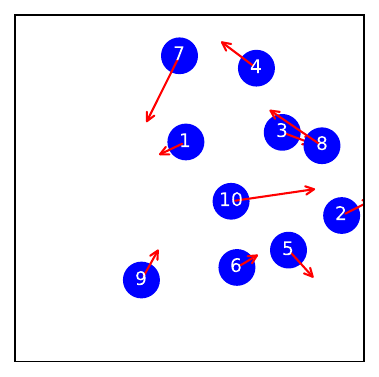}
       \caption{}
       \label{fig:bb-vis:1}
    \end{subfigure}
    \hfill
    \begin{subfigure}[b]{0.32\textwidth}
        \centering
        \includegraphics[width=1\textwidth]{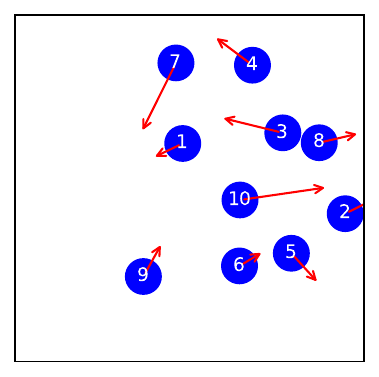}
        \caption{}
        \label{fig:bb-vis:2}
    \end{subfigure}
    \hfill
    \begin{subfigure}[b]{0.32\textwidth}
        \centering
        \includegraphics[width=1\textwidth]{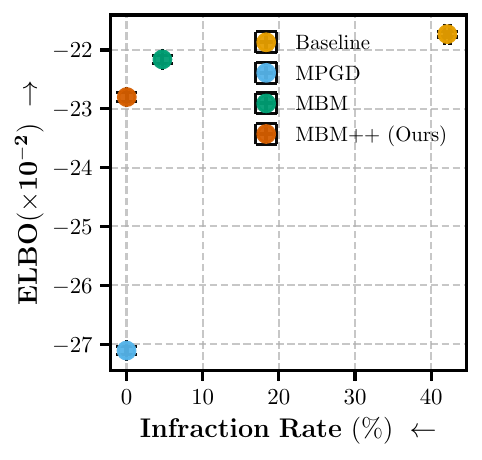}
        \caption{}
        \label{fig:bb_scatter-plot}
    \end{subfigure}
    \caption{Visualization of the bouncing balls task. Panel \subref{fig:bb-vis:1} and panel \subref{fig:bb-vis:2} show two consecutive steps before and after a bounce between the balls 3 and 8. In each step, each ball is represented by a 4-dimensional vector representing its location and velocity, shown by the red arrows. Panel \subref{fig:bb_scatter-plot} presents a Pareto comparison for diffusion-based models in terms of ELBO and infraction rate. Higher ELBO and lower infraction rate indicate better performance, with preferable methods located toward the upper-left region. Our method lies near the Pareto frontier, achieving a favorable trade off between these two metrics.}
\end{figure*}

\subsection{Traffic Scene Trajectory Simulations}

We next evaluate our method on a real-world vehicle trajectory prediction task, where the goal is to predict plausible future vehicle motions from a short observation window. Models trained to match data distributions~\citep{niedoba2024diffusion, tan2025flow, capellera2025unified} produce trajectories that violate safety or feasibility constraints, particularly in complex multi-agent driving scenarios~\citep{itra2021, ngiam2021scene, aydemir2023adapt}. In this setting, the goal is to reduce such infractions while preserving the distributional properties of vehicle models observed in the dataset. We follow the experimental setup of~\citet{lioutascritic} and conduct experiments on the INTERACTION dataset~\citep{zhan2019interaction}, which consists of vehicle trajectories collected across 11 distinct traffic scenarios. Given one second of observed vehicle states, the task is to predict the future vehicle states over a three-second horizon. This setting reflects realistic driving conditions with complex multi-agent interactions and serves as a challenging benchmark for constrained generative modeling.

We measure performance based on both realism and constraint satisfaction of the predicted trajectories. For each episode, we compute the average displacement error (ADE) and final displacement error (FDE) with respect to the ground truth trajectory. We generate six trajectory samples per episode, compute the metrics for each sample, and report the minimum ADE ($\min\mathrm{ADE}_6$) and FDE ($\min\mathrm{FDE}_6$) across samples. Additionally, we report maximum final distance (MFD) as a diversity measurement of predictions. Constraint satisfaction is assessed by reporting the collision rate and offroad rate. The collision rate measures the fraction of predicted trajectories that result in vehicle-to-vehicle overlap at every timestep, while the offroad rate measures the fraction of trajectories that violate drivable-area constraints.

We use DJINN~\citep{niedoba2024diffusion} as our baseline diffusion-based trajectory prediction model. DJINN predicts future trajectories $\rvx^{k:K}$ for all vehicles conditioned on past $k$ observed states $\rvx^{0:k}$, using a decoder-only transformer architecture~\citep{vaswani2017attention}. DJINN does not explicitly enforce constraints on offroad or collision infractions, as it is trained with a standard EDM-based diffusion objective to maximize $\log p(\rvx^{k:K} \mid \rvx^{0:k})$. In addition, we include Critic SMC~\citep{lioutascritic} as a baseline, which employs a learned critic to guide sampling via a sequential Monte Carlo procedure during generation. We train DJINN from scratch and apply MPGD (without projection) as a training-free baseline, and finally finetune DJINN with \method{}.

\begin{table*}[!t] 
  \centering
  \small
  \caption{Trajectory Prediction on INTERACTION~\citep{zhan2019interaction}: DR\_DEU\_Merging\_MT. Offroad rate and FDE are not reported for criticSMC, and we leave them blank. Our \method{} achieves significantly lower collision and offroad rate compared to prior baselines while having the lowest $\min\mathrm{ADE}_6$ and $\min\mathrm{FDE}_6$.}  
  \begin{tabular}{llllll}
    \toprule
    \centering
    Method     & Collision ($\%$)  $\downarrow$  & Offroad ($\%$) $\downarrow$ & $\min\text{ADE}_6$ $\downarrow$  & $\min\text{FDE}_6$ $\downarrow$ &  $\text{MFD}_6$ $\uparrow$ \\
    \midrule
    DJINN &  $0.39$ & $8.12$  & $0.194$ & $0.459$ & $2.483$  \\
    CriticSMC  &  $1.03$  & -- & $0.345$ & -- & $2.201$   \\
    MPGD w/o projection & $0.32$ & $0.29$ & $0.195$& $0.467$ & $2.372$ \\
    \midrule
    \method\ (Ours) & $0.27$  & $0.44$ & $0.180$ & $0.452$ & $2.352$ \\
    \bottomrule
  \end{tabular}
  \label{tab:trajectory-experiment}
\end{table*}

\begin{figure*}[!ht]
    \centering
    \begin{subfigure}[b]{0.32\textwidth}
        \centering
       \includegraphics[width=1\textwidth]{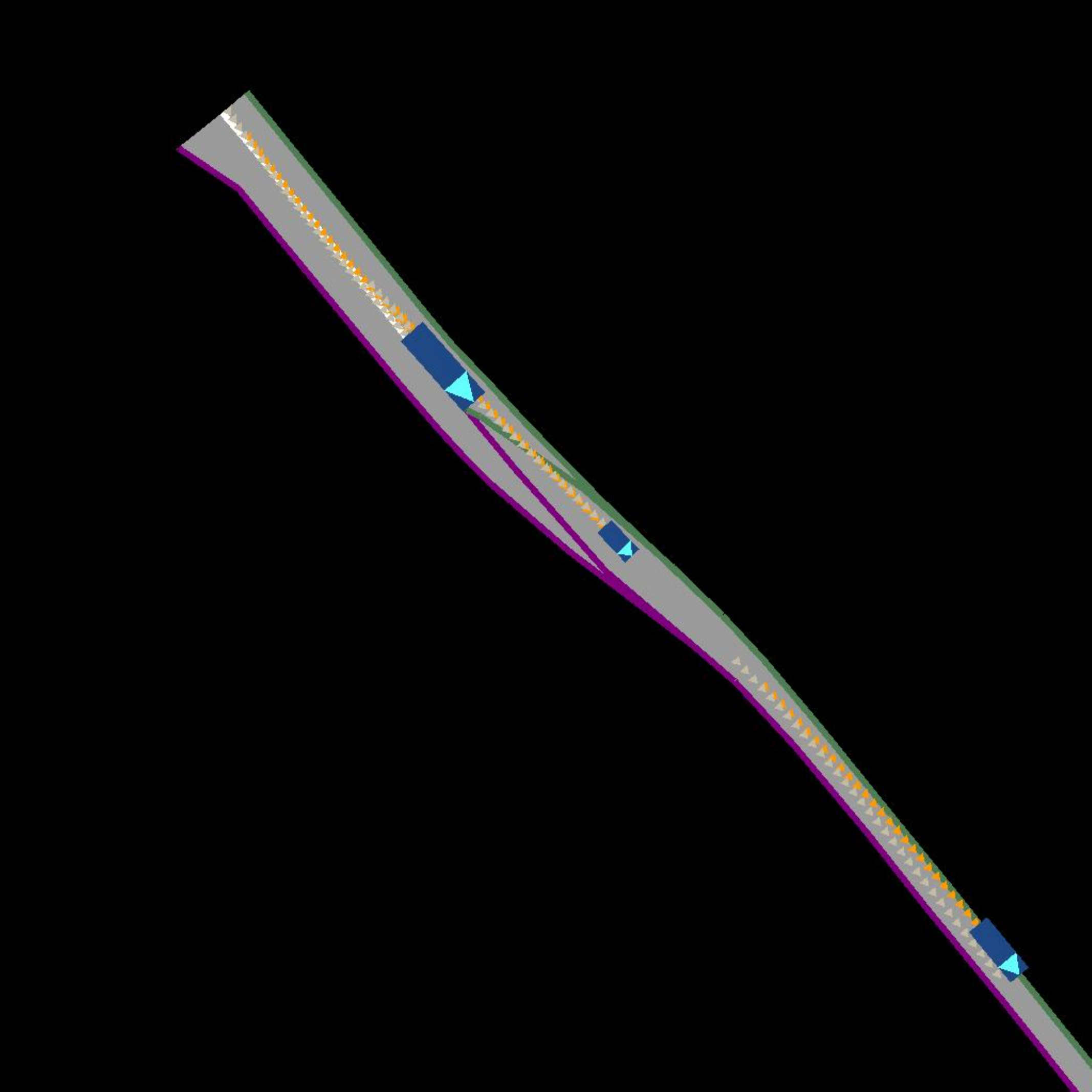}
       \caption{Standard Diffusion}
       \label{fig:baseline-single-lane}
    \end{subfigure}
    \hfill
    \begin{subfigure}[b]{0.32\textwidth}
        \centering
        \includegraphics[width=1\textwidth]{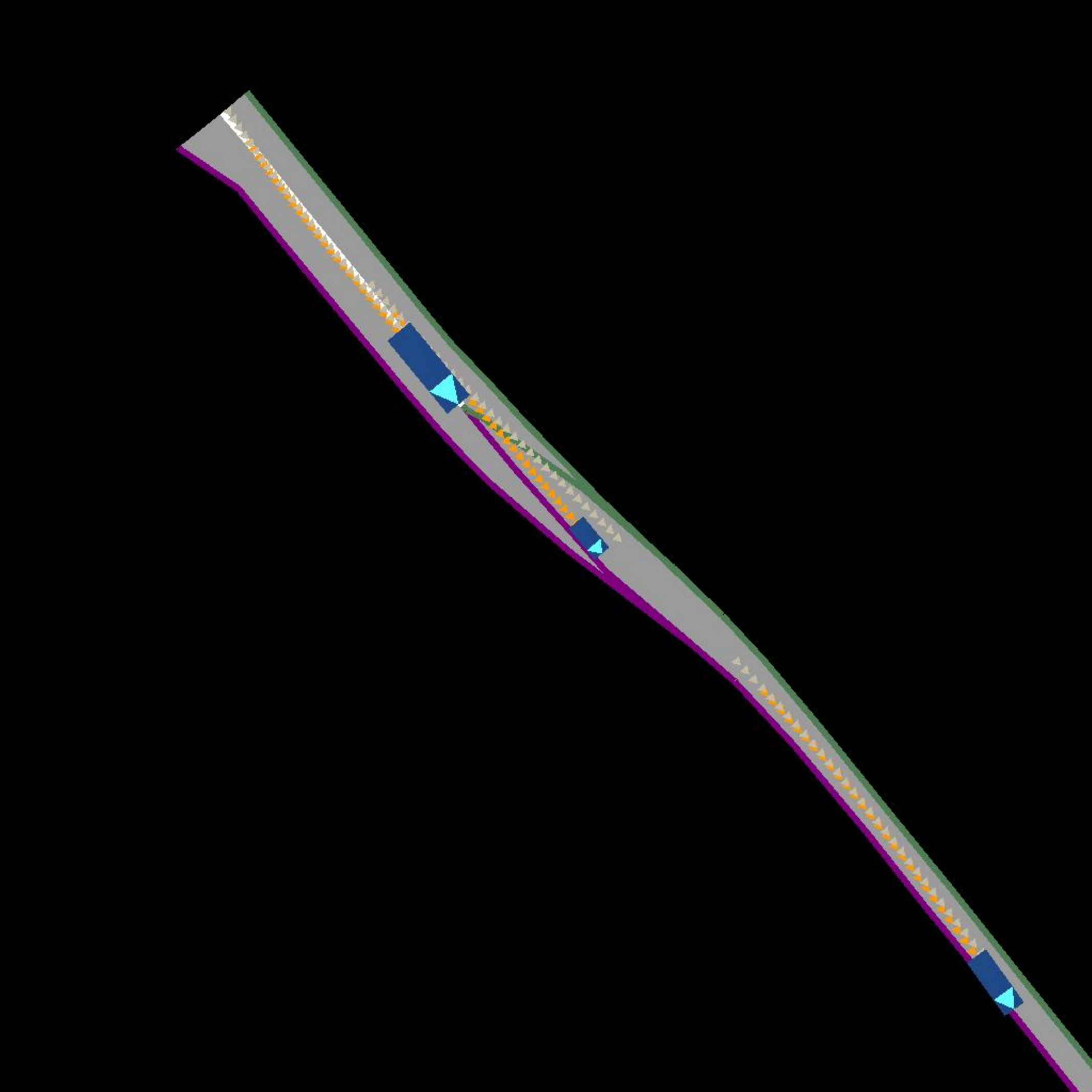}
        \caption{MPGD w/o projection}
        \label{fig:mpgd-single-lane}
    \end{subfigure}
    \hfill
    \begin{subfigure}[b]{0.32\textwidth}
        \centering
        \includegraphics[width=1\textwidth]{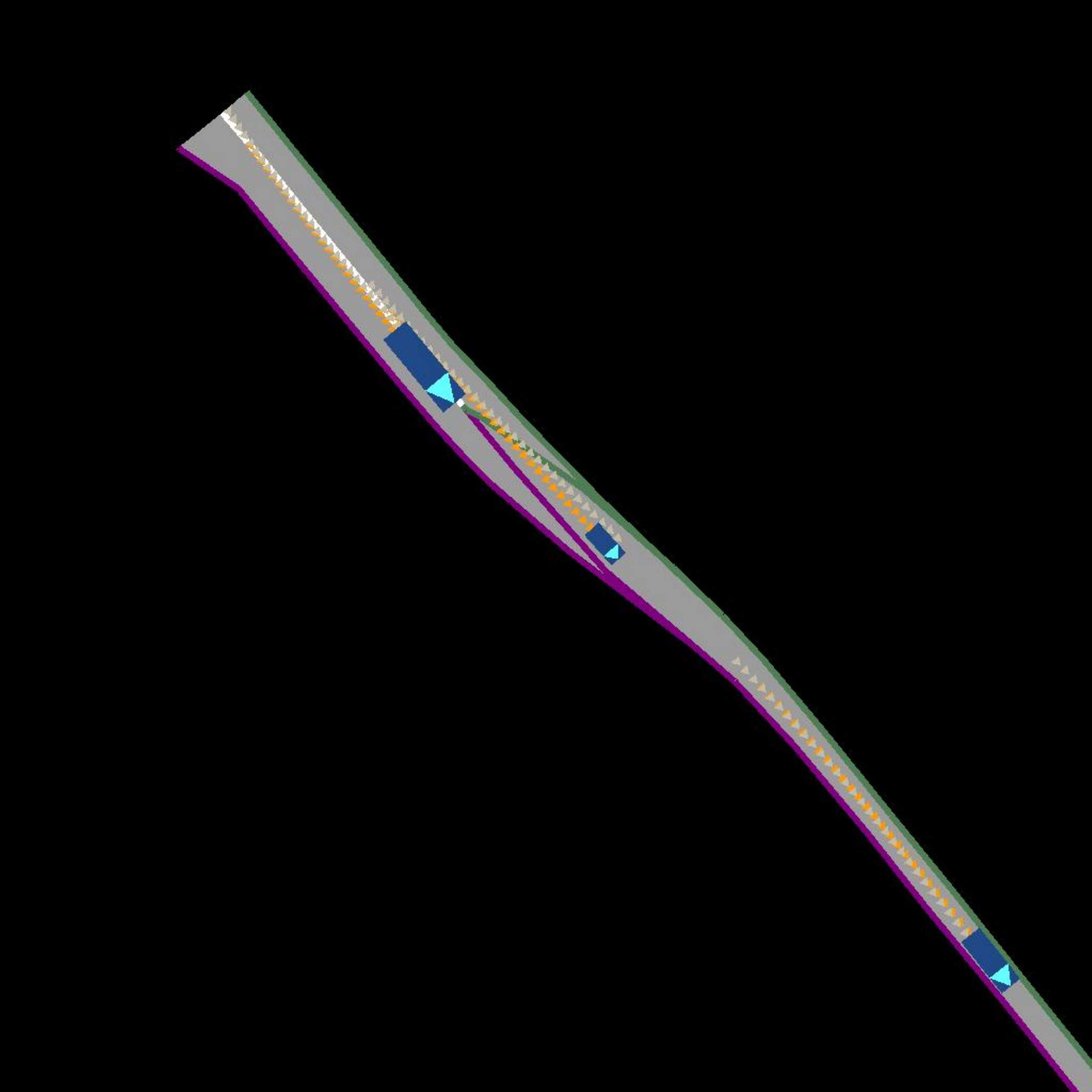}
        \caption{\method{} (Ours)}
        \label{fig:bridge-single-land}
    \end{subfigure}
    \caption{Comparisons on DR\_DEU\_Merging validation dataset (merging into single-lane scenario). In panel~\subref{fig:baseline-single-lane}, the unconstrained baseline produces an offroad trajectory for the vehicle near the bottom-right, deviating beyond the right road boundary. In panel~\subref{fig:mpgd-single-lane}, MPGD without projection applies an overly strong correction that pushes the same vehicle toward the left boundary, while the middle vehicle exhibits noticeable trajectory distortion. In contrast, panel~\subref{fig:bridge-single-land} maintains all predicted trajectories within the single drivable lane, preserving realistic motion without violations.
    }
\end{figure*}

\cref{tab:trajectory-experiment} reports results on the INTERACTION DR\_DEU\_Merging\_MT scenario. The unconstrained DJINN baseline achieves both strong average and final displacement accuracy but exhibits a high offroad rate, indicating frequent feasibility violations. Critic SMC reduces collision frequency through critic-guided particle selection, but achieves worse $\min \mathrm{ADE}_6$ (Critic SMC does not report offroad or final displacement metrics, precluding direct comparison in those metrics). Training-free guidance using MPGD substantially reduces collision rate and achieves lowest offroad rates, but it incurs an increase in displacement errors. Our method achieves the lowest collision rate while maintaining competitive offroad performance. Notably, it also attains the lowest $\min\mathrm{ADE}_6$. We speculate that incorporating offroad constraint signals during fine-tuning encourages predicted trajectories to remain closer to feasible regions, which improves alignment with ground-truth trajectories. Overall, these results suggest that constraint-aware fine-tuning can reduce infractions without sacrificing trajectory accuracy. 

\section{Related Work}
\paragraph{Generative Constraints Modeling} Several existing approaches aim to enforce constraints along the entire reverse sampling trajectory. ~\citet{christopher2024projected} adopts a stochastic gradient Langevin dynamics approach~\citep{welling2011bayesian}, projecting the noisy intermediate $\rvx_t$ onto the constraint boundary as a guidance step. ~\citet{lou2023reflected} introduces reflected diffusion models, which impose constraint conditions in the forward process, leading to substantial increase in training complexity. More recently, ~\citet{cheng2024gradient} proposes an extrapolation-correction-interpolation scheme, where unconstrained generated updates are followed by an explicit projection onto the constraint set during sampling. While effective for problems with known constraint operators, these methods are not practical in our setting: both projected and gradient-free approaches require explicit access to constraint boundaries, whereas in our setup we only have loss functions that evaluate constraint satisfaction; reflected diffusion is further limited to convex constraint sets. 

\paragraph{Guidance-based Constraint Enforcement}Beyond these methods, a broad class of data-driven conditional generation techniques incorporate constraints during sampling via classifier guidance~\citep{songscore, dhariwal2021diffusion}. Naively applying such guidance can induce distributional shift by perturbing the reverse dynamics, while prior work explicitly identifies this issue: ~\citet{naderiparizidon} shows that classifier-guided sampling can induce distributional distortion when the classifier is trained on a dataset composed of infracting and non-infracting samples whose proportion do not match the base model's infraction rate. In this case, guidance enforces feasibility but yields a distorted inference distribution, effectively pushing samples inward from the constraint boundaries. They address this by training a Bayes-optimal time-dependent classifier guidance for proportion matching. Subsequent work attributes guidance-induced quality degradation to samples drifting off the data manifold after each guided denoising step, which characterizes as a mismatch in the effective noise level. Building on this perspective, ~\citet{jung2024conditional} and~\citet{park2025temporal} both train a time predictor that estimates and corrects the noise-level mismatch, adjusting the guided samples back toward the data manifold. However, such guidance-based approaches generally require training additional classifiers or predictors beyond the base generative model.

\paragraph{Adjoint-based Sampling Methods}
Beyond Adjoint Matching for fine-tuning a pretrained base model, Adjoint Sampling~\citep{havens2025adjoint} and Adjoint Schrodinger Bridge Sampler (ASBS)~\citep{liu2025adjoint} propose strategies to accelerate adjoint-based training toward a target distribution defined up to an unnormalized energy function. Adjoint Sampling improves efficiency by simulating the controlled SDE only to obtain terminal samples, computes energy gradients at these final states, and generates intermediate states via the uncontrolled base process during inner-loop training. ASBS recasts the learning of a diffusion sampler as a Schrodinger bridge problem, thereby eliminating the memoryless schedule entirely from the existing works. It adopts an alternating optimization scheme that trains a drift model to steer trajectories toward the target distribution, with a corrector model compensates the mismatch caused by relaxing the memoryless assumption. However, both approaches assume a Brownian-motion reference distribution, which limits the ability to apply this to fine-tuning a complex or task-specific pretrained prior. 

\section{Discussion and Limitations}
This work shows that constraint-aware fine-tuning can substantially reduce infeasible samples while largely preserving the generative behavior learned by a pretrained model. By injecting constraint information through a lightweight bridge embedding rather than fully modifying the backbone parameter weights, our method strikes a favorable balance between constraint satisfaction and distributional fidelity. Empirically, this is reflected in consistently lower violation rates without degradation in likelihood or diversity observed in training-free guidance methods across both synthetic and real-world settings.

Our approach assumes that hard projection onto the constraint set is infeasible. In particular, constraints such as collision avoidance are difficult to enforce via projection without breaking the generative dynamics. As a result, while violations are significantly reduced, they may not be completely eliminated. In addition, constraint evaluation and gradient computation are required at intermediate states, which serve as surrogates for terminal states during training. Although this avoids explicit trajectory rollout and is more effective than directly evaluating constraints at noisy states, constraints that are expensive to compute may still introduce nontrivial computational overhead. Finally, the effectiveness of this denoised-state surrogate relies on assumptions, therefore, when these assumptions are weakened, constraint gradients may become less informative, potentially limiting enforcement near the boundary of the feasible region.

Since constraint enforcement in our framework is achieved through constraint gradient guidance, increasing the number of sampling steps or adopting adaptive, state-dependent guidance schedules may further reduce residual violations, potentially eliminating infractions. Future work may also explore adaptive step size near constraint boundaries, and learned surrogate constraints to improve gradient guidance quality. These directions provide promising avenues to strengthen constraint satisfaction while preserving the generative fidelity of pretrained models.

\newpage

\bibliography{refs}
\bibliographystyle{iclr2026_conference}

\newpage
\appendix
\section{Background on Flow Matching and Parameterization}
\label{sec: flow-matching}
Flow matching (FM) provides a deterministic alternative to diffusion-based generation that transports samples from a simple distribution to the data distribution. Let $\rvx_0\sim \pi$ denote a noise sample and $\rvx_1\sim p_1$ denote a data sample.  Instead of learning a score function, flow matching learns a time-dependent velocity field $v_\theta(\rvx_t, t)$ and generates samples by solving the ordinary differential equation (ODE):
\begin{align}
    d\rvx_t = v(\rvx_t, t) \,dt, \quad \rvx_0\sim \pi,
\end{align}
so that integrating the ODE from $t=0$ to $t=1$ transports noise into data. The velocity field is trained by regressing onto a prescribed target conditional flow~\citep{lipman2022flow} using
\begin{align}
    \gL(\theta) = \meanp{}{\normscaled{{v_\theta(\rvx_t, t) - \frac{d\rvx_t}{dt}}}^2}= \meanp{t, \rvx_0, \rvx_1}{\norm{v_\theta(\rvx_t, t) - (\rvx_1 - \rvx_0)}^2},
\end{align}
where $\rvx_t = t \rvx_1 + (1-t) \rvx_0$, a linear interpolation between noise and data.

\subsection{Express the Marginal Score in terms of Posterior Mean $\meanp{\rvx_0 | \rvx_t}{\rvx_0}$}
\label{subsec: posterior-mean}
Show that $\nabla_{\rvx_t} \log p_t(\rvx_t) = \frac{\meanp{\rvx_0 | \rvx_t}{\rvx_0} - \rvx_t}{\sigma^2(t)}$.
\begin{proof}
\begin{align}
        \nabla_{\rvx_t} \log p_t(\rvx_t) &= \frac{\nabla_{\rvx_t} p_t(\rvx_t)}{p_t(\rvx_t)} 
        = \frac{\nabla_{\rvx_t}\int p_t(\rvx_t | \rvx_0) p(\rvx_0)\,d\rvx_0}{p_t(\rvx_t)} \nonumber\\
        &= \frac{\int\nabla_{\rvx_t}\frac{1}{\sqrt{2\pi}t} \exp\left(-\frac{(\rvx_t - \rvx_0)^2}{2\sigma^2(t)}\right)p(\rvx_0)\,d\rvx_0}{p_t(\rvx_t)} \nonumber\\
        &= \frac{\int p_t(\rvx_t | \rvx_0)\left(-\frac{\rvx_t - \rvx_0}{\sigma^2(t)}\right)p(\rvx_0)\,d\rvx_0}{p_t(\rvx_t)} \nonumber \\
        &= \frac{\int p_t(\rvx_t | \rvx_0) p(\rvx_0) \rvx_0\,d\rvx_0 - \int p_t(\rvx_t | \rvx_0) p(\rvx_0) \rvx_t\,d\rvx_0}{\sigma^2(t) p_t(\rvx_t)} \nonumber\\
        &= \frac{p_t(\rvx_t)\int p_t(\rvx_0 | \rvx_t)\rvx_0 \,d\rvx_0 - \rvx_t \int p_t(\rvx_t, \rvx_0)\,d\rvx_0}{\sigma^2(t) p_t(\rvx_t)}\nonumber \\
        &= \frac{\int p_t(\rvx_0 | \rvx_t) \rvx_0\,d\rvx_0 - \rvx_t}{\sigma^2(t)}\nonumber \\
        &= \frac{\meanp{\rvx_0 | \rvx_t}{\rvx_0} - \rvx_t}{\sigma^2(t)}\label{eq: posterior-mean}.
\end{align}
\end{proof}
\subsection{Parameterization of Flow Matching}
Similarly, instead of directly learning a marginal velocity field, flow matching defines the velocity as the conditional expectation $v_t(\rvx_t) = \meanp{}{\dot\rvx_t\mid \rvx_t}$. For the linear interpolation path $\rvx_t = t\rvx_1 + (1-t)\rvx_0$, the time derivative is $\dot\rvx_t = \rvx_1 - \rvx_0$. Expressing $\rvx_0 = \frac{\rvx_t - t\rvx_1}{1-t}$ and taking the conditional expectation yields
\begin{align}
    \label{eq: fm-param}
    v_t(\rvx_t) = \frac{\meanp{}{\rvx_1\mid \rvx_t} - \rvx_t}{1 - t},
\end{align}
such that the learned flow network learns a denoised endpoint approximating $\meanp{}{\rvx_1\mid \rvx_t}$ then derives the velocity field. This shared conditional mean perspective provides a unified abstraction for both diffusion and flow matching models.

\subsection{Bridge Conditioning on Flow Matching}
The same bridge conditioning mechanism extends naturally to flow-matching models. In particular, the flow-matching model can be viewed as learning a conditional estimate of the clean data state through its parameterized velocity, consistent with the conditional mean formulation introduced in~\cref{eq: fm-param}. As a result, the same bridge-conditioned denoised estimate $D_{\theta, \phi}^\Omega(\rvx_t; t)$ is shared across diffusion and flow-matching models with the objective function:
\begin{align}
    \gL^{\mathrm{FM}}(\phi) &= \meanp{}{\normscaled{v_{\theta, \phi}^\Omega(\rvx_t; t) - v^*(\rvx_t, t \mid \rvx_0, \rvx_1)}^2}, \nonumber \\
    \mathrm{where}\ &v_{\theta, \phi}^\Omega(\rvx_t; t) = \frac{D_{\theta, \phi}^\Omega(\rvx_t; t) - \rvx_t}{1-t}.
\end{align}

Using this formulation, the bridge-conditioning velocity is obtained directly from the denoised estimate, and the flow-matching objective is defined accordingly. As a result, the bridge-conditioning mechanism applies uniformly across both paradigms: the constraint guidance signals are evaluated on the denoised estimate, injected into the model through the lightweight bridge embedding, and the resulting model is trained under the corresponding objective. This shared denoised-state formulation enables a unified fine-tuning framework across diffusion and flow matching models.

\section{Proof for~\cref{theorem: theorem}}
\label{sec: proof}
\begin{proof}
    Given assumption forward kernel concentrates and denoising consistency, 
    \begin{align}
        \lim_{t\downarrow 0}\norm{D_\theta(\rvx_t; t) - \rvx_0}\leq \lim_{t\downarrow 0}\norm{D_\theta(\rvx_t; t) - \rvx_t} + \lim_{t\downarrow 0}\norm{\rvx_t - \rvx_0} \xrightarrow{p} 0. 
    \end{align}
    By smoothness,
    \begin{align}
    \label{eq: smoothness}
        \lim_{t\downarrow 0}\norm{\nabla_{\rvx_t} \ell^\Omega(\rvx_t) - \nabla_{\rvx_t} \ell^\Omega(D_\theta(\rvx_t; t))} \leq L\lim_{t\downarrow 0}\norm{\rvx_t - D_\theta(\rvx_t; t)}\xrightarrow{p} 0 \quad\mathrm{as}\ t\downarrow 0.
    \end{align}
    By the chain rule:
    \begin{align}
        \nabla_{\rvx_t} \ell^\Omega(D_\theta(\rvx_t; t)) = \nabla_{D_\theta(\rvx_t; t)}\ell^\Omega(D_\theta(\rvx_t; t)) \frac{\partial D_\theta(\rvx_t; t)}{\partial \rvx_t},
    \end{align}
    then
    \begin{align}
    \label{eq: jacobian-id}
        &\lim_{t\downarrow 0}\norm{\nabla_{\rvx_t} \ell^\Omega(D_\theta(\rvx_t; t)) - \nabla_{D_\theta(\rvx_t; t)} \ell^\Omega(D_\theta(\rvx_t; t))}\nonumber \\
        =&\ \lim_{t\downarrow 0}\normscaled{\left(\frac{\partial D_\theta(\rvx_t; t)}{\partial \rvx_t} - \mI\right) \nabla_{D_\theta(\rvx_t; t)} \ell^\Omega(D_\theta(\rvx_t; t))} \nonumber \\
        \leq &\ \lim_{t\downarrow 0} \left(\normscaled{\frac{\partial D_\theta(\rvx_t; t)}{\partial \rvx_t} - \mI}\normscaled{\nabla_{D_\theta(\rvx_t; t)} \ell^\Omega(D_\theta(\rvx_t; t))}\right) \nonumber\\
        \xrightarrow{p}&\ 0
    \end{align}
    Therefore,
    \begin{align}
        &\lim_{t\downarrow 0} \normscaled{\nabla_{\rvx_t} \ell^\Omega(\rvx_t) - \nabla_{D_\theta(\rvx_t; t)} \ell^\Omega(D_\theta(\rvx_t; t))} \nonumber \\
        \leq &\ \lim_{t\downarrow 0} \left(\normscaled{\nabla_{\rvx_t} \ell^\Omega(\rvx_t) - \nabla_{\rvx_t}\ell^\Omega(D_\theta(\rvx_t; t))} + \normscaled{\nabla_{\rvx_t} \ell^\Omega(D_\theta(\rvx_t; t)) - \nabla_{D_\theta(\rvx_t; t)} \ell^\Omega(D_\theta(\rvx_t; t))}\right) \nonumber\\
        \leq&\ \lim_{t\downarrow 0} 0
    \end{align}
    by combining~\cref{eq: smoothness} and ~\cref{eq: jacobian-id}.
\end{proof}

\subsection{Complementary Theorems}
\begin{theorem}
    (\textbf{Near-optimal DSM implies small path-space KL}) Let $\sP_{\theta, \phi}^\Omega$ and $\sP^*$ be the path measures induced by the following two reverse-time SDEs on $C([0, T]; \sR^d)$, sharing the same diffusion coefficient $\sigma(t)$:
    \begin{align}
        \sP_{\theta, \phi}^\Omega&: d\rvx_t = (f(\rvx_t; t) - \sigma^2(t) s^{\Omega}_{\theta, \phi}(\rvx_t; t))\,dt + \sigma(t)d\,\rvw \nonumber\\
        \sP^* &: d\rvx_t = \left(f(\rvx_t; t) - \sigma^2(t) s^*(\rvx_t; t)\right)\,dt + \sigma(t)\,d\rvw
    \end{align}
    Define $u(\rvx_t; t) = \sigma(t) (s_{\theta, \phi}^\Omega(\rvx_t; t) - s^*(\rvx_t; t))$. Assume Novikov's condition holds for $u$ under $\sP^*$, 
    \begin{align}
        \meanp{\sP^*}{\exp\left(\frac 12 \int_0^T \norm{u(\rvx_t; t)}^2\,dt\right)} < \infty.
    \end{align}
    Then the path-space KL divergence admits the identity:
    \begin{align}
        D_{\mathrm{KL}} (\sP_{\theta, \phi}^\Omega\parallel\sP^*) = \frac 12 \meanp{\sP_{\theta, \phi}^\Omega}{\int_0^T \sigma^2(t) \norm{s_{\theta, \phi}^\Omega(\rvx_t; t) - s^*(\rvx_t; t)}^2\,dt}.
    \end{align}
    Moreover, assuming the learned score satisfies the near-optimality condition for some small $\epsilon > 0$,
    \begin{align}
        \meanp{t\sim \pi(t), \rvx_t}{\norm{s_{\theta, \phi}^\Omega(\rvx_t; t) - s^*(\rvx_t; t)}^2}\leq \epsilon,
    \end{align}
    Then
    \begin{align}
        D_{\mathrm{KL}} (\sP_{\theta, \phi}^\Omega\parallel \sP^*) \leq \frac 12\sup_{t} \frac{1}{\pi(t)} \sigma(t)\epsilon.
    \end{align}
\end{theorem}

\begin{theorem}
    (\textbf{Asymptotic concentration of probability mass on} $\Omega$)
    Given $\lim_{t\to T} \gamma(t) = 0$, $\lim_{t\to 0} \gamma(t) \to \infty$, and $C^1$ non-negative loss function $\ell^\Omega(\rvx)$, then
    \begin{align}
        \exp(-\gamma(t) \ell^\Omega(D_\theta(\rvx_t; t)))\xrightarrow{p} \1_\Omega(\rvx_t), \quad \mathrm{as}\ t\to 0
    \end{align}
    so $p^\Omega(\rvx_t; t)\propto p(\rvx_t; t)\exp(-\gamma(t) \ell^\Omega(D_\theta(\rvx_t; t)))$ converges to $p(\rvx_t; t)\1_{\Omega}(\rvx_t)$ at terminal state, which places mass on a truncation of $p(\rvx_t)$ to $\Omega$. This follows the argument of MBM~\citep{naderiparizi2025constrained}.
\end{theorem}

\section{Connection to SOC in Constrained Generation}
\label{sec: connection-to-soc}
Stochastic optimal control (SOC) studies the problem of modifying stochastic dynamics in order to reduce task-specific costs, while preserving the base process behavior. In a standard SOC formulation, the system evolves according to controlled stochastic dynamics $\sP^u$ of the form, in contrast to base stochastic dynamics $\sP^{\mathrm{base}}$:
\begin{align}
    \sP^{\mathrm{base}} &: d\rvx_t = b(\rvx_t, t)\,dt + \sigma(t)\,d\rvw, \\
    \sP^u &: d\rvx_t^u = (b(\rvx_t^u, t) + \sigma(t) u(\rvx_t^u, t))\,dt + \sigma(t)\,d\rvw, \label{eq: controlled-sde}
\end{align}
where $b: \sR^d\times [0, T]\to \sR^d$ denotes the base pretrained drift, and $u: \sR^d\times [0, T]\to \sR^d$ is the control policy drawn from a set of admissible controls $\gU$. The objective is to optimize the trade-off between the control and control effect, discouraging excessive deviation from the base dynamics while achieving low control by optimizing the cost functional:
\begin{align}
\label{eq: am-objb}
    V(\rvx_t; t) = \min_{u\in \gU} J(u; \rvx_t, t) = \min_{u\in \gU} &\meanp{}{\int_0^T \left(\frac 12 \norm{u(\rvx_t^u, t)}^2 + f(\rvx_t^u, t)\right)\,dt + g(\rvx_T^u)},
\end{align}
where $f$ denotes the running state cost, evaluated at each intermediate state, and $g$ denotes the terminal state cost, evaluated at the final state $\rvx_T$. Notice that the accumulated control cost $\int \frac 12 \norm{u(\rvx_t^u, t)}^2\,dt$ is exactly the KL divergence between the controlled and uncontrolled path measure, denoted as $D_{\mathrm{KL}}(\sP^u\parallel \sP^{\mathrm{base}})$ by Girsanov's theorem.

The constrained generation problems studied in our experiments can be naturally interpreted through the lens of stochastic optimal control (SOC). We assume a pretrained generative model that already approximates the data distribution well, but may produce samples that violate task-specific constraints. This pretrained model defines a base score $s_\theta(\cdot; \cdot)$, analogous to the base dynamics $b(\cdot)$ in SOC. Rather than relearning the data distribution, our goal is to incorporate constraint information in a way that minimally perturbs these base dynamics. 

\subsection{Adjoint Matching and Constrained Fine-tuning}
Adjoint matching~\citep{domingo2024adjoint} provides a general mechanism for solving such SOC problems under cost-affine control formulations, and applies broadly to controlled diffusion systems regardless of the specific form of the objective.

Classical results~\citep{kappen2005path, domingo2024stochastic} show that the optimal control is explicitly determined by the gradient of the optimal value function $V(\rvx_t; t) = \min_{u\in \gU} J(u; x, t)$,
\begin{align}
    u^*(\rvx_t, t) = -\sigma(t)^\top \nabla_{\rvx_t} V(\rvx_t, t),
\end{align} 
where the optimal value function admits an expectation-based formulation with respect to the uncontrolled base dynamics. This reveals an inherent coupling: computing the optimal control requires knowledge of the value function, while the value function itself depends on the controlled dynamics induced by the optimal control, so learning $u$ without an accurate value estimate generally yields only locally optimal solutions. This coupling underlies the difficulty of learning optimal controls for solving SOC problem in practice.

In addition, evaluating the value function requires integrating the running cost $f(\cdot; \cdot)$ along trajectories $\{\rvx_t\}_{t=0}^T$ of the controlled SDE and computing the terminal cost $g(\cdot)$ at the final state $\rvx_T$. Since both intermediate and terminal states depend on the learned control term $u_\theta(\rvx_t; t)$, training requires sampling full trajectories from the controlled dynamics and to compute the gradient of the objective:
\begin{align}
\label{eq: direct-diff}
    \nabla_\theta \gL(u_\theta; \bar\rvx) = \nabla_\theta \left(\int_0^T \left(\frac 12 \norm{u_\theta(\rvx_t, t)}^2 + f(\rvx_t; t)\right)\,dt + g(\rvx_T)\right) := \nabla_{\theta} J(u; \rvx_0, 0)
\end{align}
Directly differentiating through these simulations in~\cref{eq: direct-diff} is computationally prohibitive, and adjoint matching therefore adopts an adjoint method, which is analogues to the method proposed from neural ODE, bypassing backpropagation through the simulation. The adjoint state is defined the gradient of the remaining cost-to-go with respect to the state $\bar\rvx_t = \{\rvx_t\}_{t}^T$:
\begin{align}
    a(t; \bar\rvx_t, u_\theta) = \nabla_{\rvx_t} \left(\int_t^1 \left(\frac 12\norm{u_\theta(\rvx_{t'}; t')}^2+ f(\rvx_{t'}; t') \right)\,dt + g(\rvx_T)\right) := \nabla_{\rvx_t} J(u; \rvx_t, t),
\end{align}
satisfying adjoint ODE
\begin{align}
    \label{eq: adjoint-ode}
    a(t; \bar\rvx_t, u_\theta) &= -\left[a(t; \bar\rvx_t, u_\theta)^\top (\nabla_{\rvx_t} (b(\rvx_t; t) + \sigma(t) u(\rvx_t; t))) + \nabla_{\rvx_t} (f(\rvx_t; t) + \frac 12 \norm{u_\theta(\rvx_t; t)}^2)\right] \nonumber \\
    a(T; \bar\rvx_T, u_\theta) &= \nabla_{\rvx_T} g(\rvx_T)
\end{align}
Under the assumption that the control is optimal $u^*(\rvx_t; t) = \meanp{\sP^*}{-\sigma(t)^\top a(t; \bar\rvx_t, u^*)\mid \rvx_t}$, the adjoint dynamics simplify to~\citet{domingo2024adjoint}
\begin{align}
\label{eq: lean-adj-ode}
    \frac{d}{dt}a(t; \bar\rvx_t) = -(a(t; \rvx_t)^\top \nabla_{\rvx_t} b(\rvx_t; t) + \nabla_{\rvx_t} f(\rvx_t; t)), \quad a(T; \bar\rvx_T) = \nabla_{\rvx_T} g(\rvx_T).
\end{align}

The training procedure is as follows:
\begin{itemize}
    \item Forward rollout: sample trajectories by rolling out the controlled SDE~\cref{eq: controlled-sde} forward in time. The control $u_\theta$ is parameterized implicitly as the difference between the fine-tuned model and the pretrained base model. At initialization, the fine-tuned model is identitcal to the pretrained model, and therefore $u_\theta = 0$.
    \item Backward adjoint solving: given the sampled trajectory, solve the adjoint ODE~\cref{eq: lean-adj-ode} backward in time from the terminal state.
    \item Compute the adjoint matching objective:
    \begin{align}
        \gL_{\mathrm{adj-match}}(\theta) = \sum_t \normscaled{u_\theta(\rvx_t; t) + \sigma(t)^\top a(t)}^2,
    \end{align}
    and update $\theta$ via gradient descent using $\nabla_\theta \gL(\theta)$.
\end{itemize}

Under this formulation, the control term is not learned explicitly, but instead emerges implicitly through fine-tuning the full model parameters so that the learned dynamics absorb the control effect, with exponential moving average (EMA) used to stabilize training. In contrast, we freeze the pretrained backbone and introduce a lightweight bridge embedding as the only trainable component. As described in~\cref{sec: bridge-finetune}, we leave the pretrained backbone parameters unchanged and absorb the constraint-induced control through the bridge embedding. This design enforces feasibility with minimal perturbation to the distribution induced by the base model, avoiding full model fine-tuning while providing a stable and parameter-efficient alternative to adjoint matching.

While adjoint matching avoids storing the full computation graph of the numerical solver, it still requires to solve the controlled SDE~\cref{eq: controlled-sde} from $t=0$ to $T$ and an adjoint ODE~\cref{eq: lean-adj-ode} from $t=T$ to $0$, where the adjoint dynamics evaluate the running cost $f(\cdot; \cdot)$ and its states gradient along the entire trajectory in~\cref{eq: lean-adj-ode}, leading to inherently sequential and time-consuming training. In practice, evaluating running cost $f$ at every intermediate state is already computationally expensive, and computing its state gradient further exacerbates this cost. Consequently the practical algorithm in Adjoint Matching omits this term for efficiency.

Our constrained fine-tuning approach follows the same high-level objective of preventing infracting samples while preserving the baseline generative behavior, but avoids simulation and adjoint backpropagation during training. Instead of estimating costs by integrating along controlled SDE trajectories and backpropagating through the resulting adjoint dynamics, we apply constraint costs directly to a one-step denoised estimate $D_\theta(\rvx_t; t)$ at intermediate noisy state, using it as a surrogate for the terminal state . This eliminates running cost evaluation and backpropagation of entire trajectory, requiring only the evaluation of a surrogate terminal loss $\ell^\Omega(D_\theta(\rvx_t; t))$, and substantially reducing computational and memory overhead. As a result, our method achieves more efficient training, and empirically produces higher-quality non-infracting samples while remaining closely.

\section{More details of the Experiments}
\subsection{Bouncing Balls}
\subsubsection{Training Details}
\paragraph{Architecture} The architectures are based on that of \citet{zhang2022motiondiffuse} which is a transformer-based \citep{vaswani2017attention} architecture for human motion generation. In contrast to our experiment which involves multiple agents (i.e., the 10 balls in each sample), human motion data only involves one agent. Therefore, we modify the architecture of \citet{zhang2022motiondiffuse} by adding a cross-attention layer after each self-attention layer so that it can model the dependence of each agent on the other ones at the same time step. Furthermore, to save computational cost, we replace the temporal attention layers with Longformer \citep{beltagy2020longformer} with a window size of 10. In our preliminary results Longformer led to a faster convergence. In order to condition on the bridge input we simply apply an MLP layer to the bridge input and add it to the first layer hidden states of the model.

\paragraph{Diffusion Process} We use EDM~\citep{karras2022elucidating} framework and change the diffusion schedule to log-linear. In this experiment we choose $\sigma_{\mathrm{min}} = 3\times 10^{-5}$ and $\sigma_{\mathrm{max}} = 80$.

\paragraph{Flow Matching Process} We follows the standard flow matching practice~\citep{lipman2024flow} with a linear interpolation probability path between noise and data. Instead of directly predicting the velocity field, the model is parameterized to predict the clean data, which implicitly defines the corresponding velocity through the interpolation path with ~\cref{eq: fm-param}. During training, the interpolation time $t$ is sampled from a logit-normal distribution with mean $-0.6$ and standard deviation $1.6$.

\paragraph{Training} We first train both standard diffusion and flow matching models for 950,000 iterations with a learning rate of $3 \times 10^{-4}$ on a Tesla V100 GPU which takes $70$ hours respectively. All other models are fine-tuned from this checkpoint for 950,000 iterations with a learning rate of $3 \times 10^{-5}$, which takes 72 hours respectively on a Tesla P100 GPU. For all models we used Adam optimizer with a batch size of 32.

\paragraph{Diffusion Sampling} We use the first-order Euler-Maruyama sampler~\citep{songscore} with a log-linear schedule of $\sigma$, 200 time steps, and $S_{\text{churn}} = 10$.

\paragraph{Flow Matching Sampling} We use first-order Euler ODE solver with a uniform schedule of $t$, 50 time steps.

\subsection{Traffic Scene Trajectory}
\subsubsection{Training Details}
\paragraph{Architecture} This baseline architecture is from~\citet{niedoba2024diffusion}, whose architecture is transformer-based~\citep{vaswani2017attention}, which is composed of an MLP encoder, a stack of self-attention residual blocks to enable the interactions between vehicles and cross-attention residual blocks to accommodate the interactions between vehicles and CNN-embedded~\citep{lecun1995convolutional} road representations, and a MLP decoder. The original data input is of the shape $[B, A, T, F]$ corresponding to $A$ vehicles with $F$ features rolling $T$ trajectory steps in a batch of $B$ traffic scenes. In this experiment by convention, $T = 40$. For the fine-tuning model \method{}, the architecture has extra two-layer MLP embeddings for the gradients of collision and offroad losses. 
\paragraph{Diffusion Process} We use EDM~\citep{karras2022elucidating} framework and change the diffusion schedule to log-linear. In this experiment we choose $\sigma_{\mathrm{min}} = 2\times 10^{-4}$ and $\sigma_{\mathrm{max}} = 80$.
\paragraph{Training} We first train a standard diffusion model with a learning rate of $3 \times 10^{-4}$ on a Tesla V100 GPU which takes 144 hours. The~\method{} is fine-tuned based on the standard diffusion model for 72 hours with a learning rate of $3 \times 10^{-5}$ on two NVIDIA L40s GPUs . For all models we used Adam optimizer with a batch size of 32.
\paragraph{Sampling} We use the first-order Euler-Maruyama sampler~\citep{songscore} with a log-linear schedule of $\sigma$, 200 time steps, and $S_{\text{churn}} = 10$.

\subsubsection{More Experimental Results}
\begin{table*}[!h] 
  \centering
  \small
  \caption{Trajectory Prediction on INTERACTION~\citep{zhan2019interaction}: DR\_DEU\_Roundabout\_OF.}  
  \begin{tabular}{llllll}
    \toprule
    \centering
    Method     & Collision ($\%$)  $\downarrow$  & Offroad ($\%$) $\downarrow$ & $\min\text{ADE}_6$ $\downarrow$  & $\min\text{FDE}_6$ $\downarrow$ &  $\text{MFD}_6$ $\uparrow$ \\
    \midrule
    DJINN &  $1.48$ & $22.03$  & $0.324$ & $1.043$ & $5.266$  \\
    CriticSMC &  $0.08$  & -- & $0.445$ & -- & $3.425$   \\
    MPGD w/o projection & $1.62$ & $1.78$ & $0.338$ & $1.000$ & $4.787$ \\
    \midrule
    \method\ (Ours) & $1.57$  & $2.49$ & $0.327$ & $1.005$ & $4.684$ \\
    \bottomrule
  \end{tabular}
  \label{}
\end{table*}

\begin{table*}[!h] 
  \centering
  \small
  \caption{Trajectory Prediction on INTERACTION~\citep{zhan2019interaction}: DR\_USA\_Intersection\_MA.}  
  \begin{tabular}{llllll}
    \toprule
    \centering
    Method     & Collision ($\%$)  $\downarrow$  & Offroad ($\%$) $\downarrow$ & $\min\text{ADE}_6$ $\downarrow$  & $\min\text{FDE}_6$ $\downarrow$ &  $\text{MFD}_6$ $\uparrow$ \\
    \midrule
    DJINN &  $1.24$ & $1.25$  & $0.231$ & $0.637$ & $3.990$  \\
    CriticSMC  &  $0.09$  & -- & $0.448$ & -- & $2.871$   \\
    MPGD w/o projection & $1.13$ & $0.04$ & $0.229$& $0.635$ & $3.882$ \\
    \midrule
    \method\ (Ours) & $0.92$  & $0.06$ & $0.211$ & $0.629$ & $3.834$ \\
    \bottomrule
  \end{tabular}
  \label{}
\end{table*}

\begin{table*}[!h] 
  \centering
  \small
  \caption{Trajectory Prediction on INTERACTION~\citep{zhan2019interaction}: DR\_USA\_Roundabout\_FT.}  
  \begin{tabular}{llllll}
    \toprule
    \centering
    Method     & Collision ($\%$)  $\downarrow$  & Offroad ($\%$) $\downarrow$ & $\min\text{ADE}_6$ $\downarrow$  & $\min\text{FDE}_6$ $\downarrow$ &  $\text{MFD}_6$ $\uparrow$ \\
    \midrule
    DJINN &  $1.10$ & $8.50$  & $0.236$ & $0.704$ & $3.757$  \\
    CriticSMC  &  $0.07$  & -- & $0.444$ & -- & $2.974$   \\
    MPGD w/o projection & $1.13$ & $0.26$ & $0.233$& $0.685$ & $3.594$ \\
    \midrule
    \method\ (Ours) & $1.12$  & $0.35$ & $0.219$ & $0.676$ & $3.544$ \\
    \bottomrule
  \end{tabular}
  \label{}
\end{table*}

\begin{figure*}[!th]
    \centering
    \begin{subfigure}[b]{0.32\textwidth}
        \centering
       \includegraphics[width=1\textwidth]{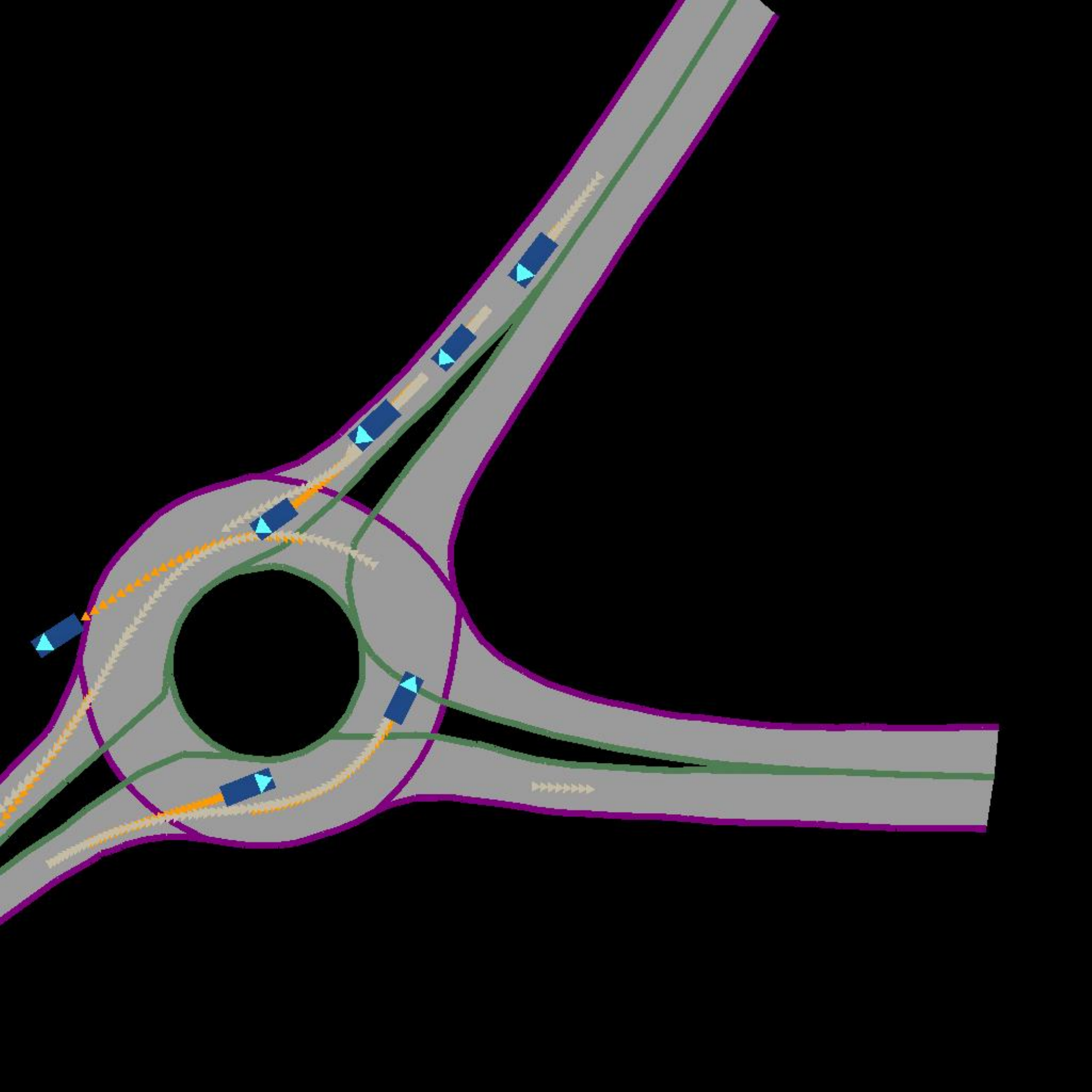}
       \caption*{Standard Diffusion}
        \label{}
    \end{subfigure}
    \hfill
    \begin{subfigure}[b]{0.32\textwidth}
        \centering
        \includegraphics[width=1\textwidth]{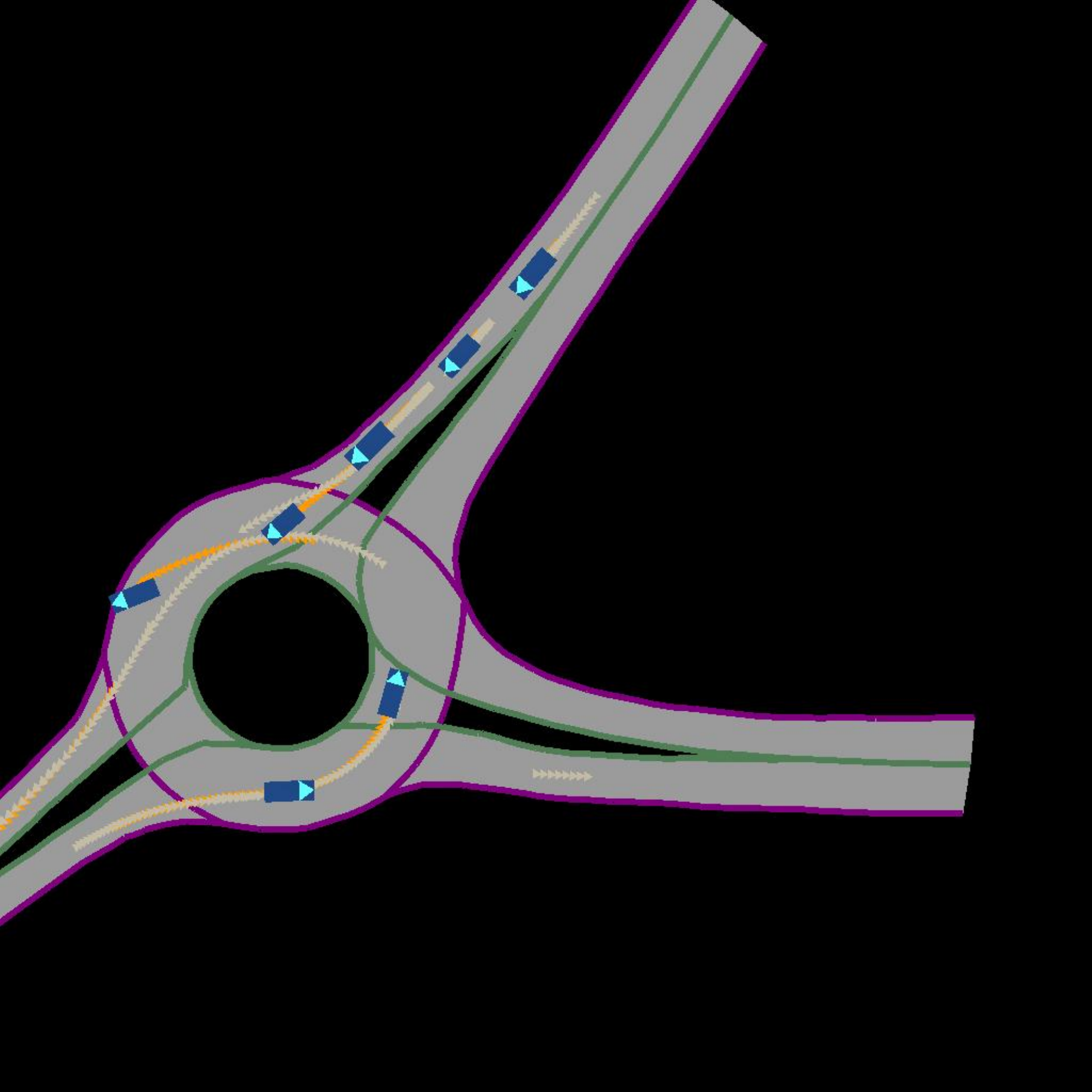}
        \caption*{MPGD w/o projection}
        \label{}
    \end{subfigure}
    \hfill
    \begin{subfigure}[b]{0.32\textwidth}
        \centering
        \includegraphics[width=1\textwidth]{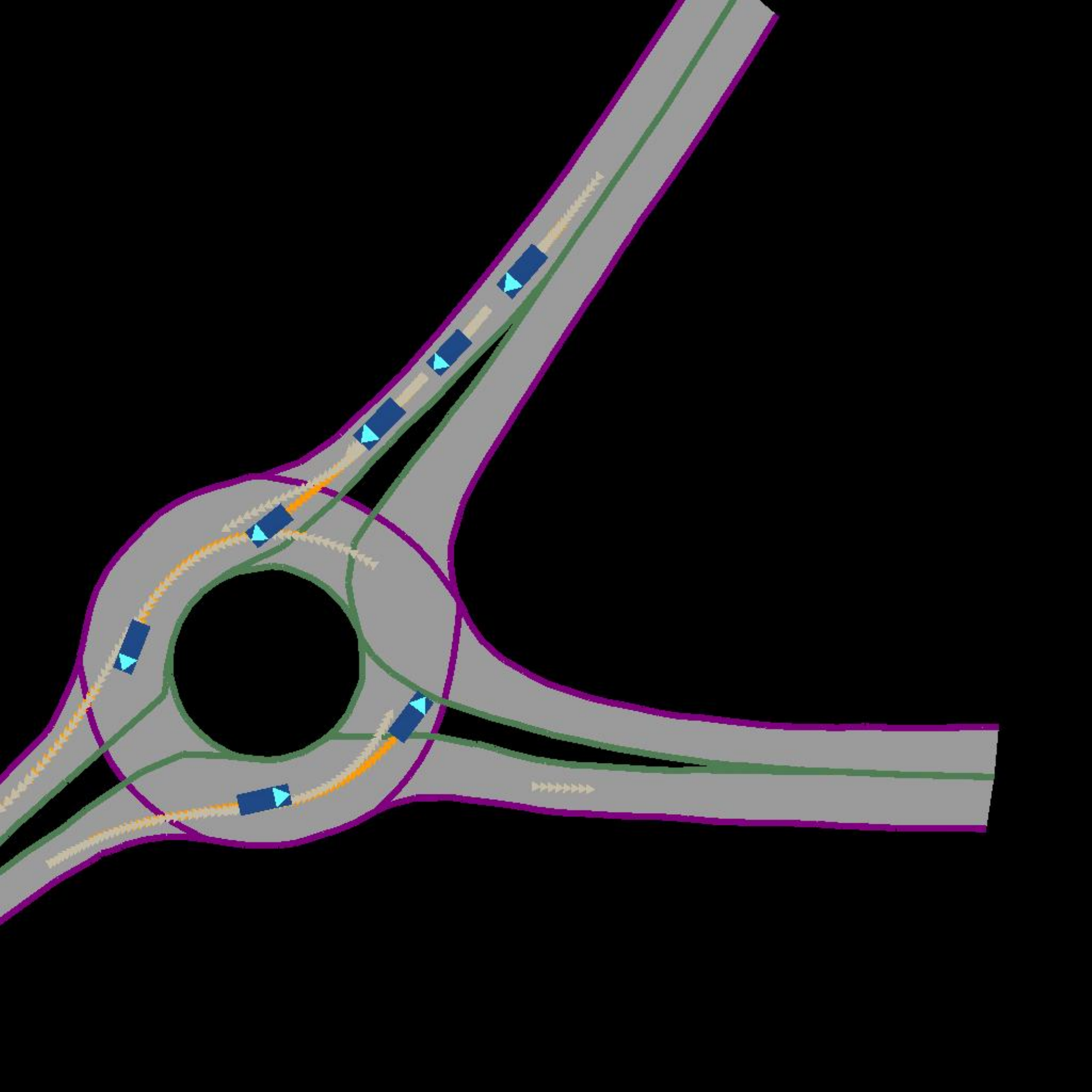}
        \caption*{\method{} (Ours)}
        \label{}
    \end{subfigure}
\end{figure*}

\begin{figure*}[!th]
    \centering
    \begin{subfigure}[b]{0.32\textwidth}
        \centering
       \includegraphics[width=1\textwidth]{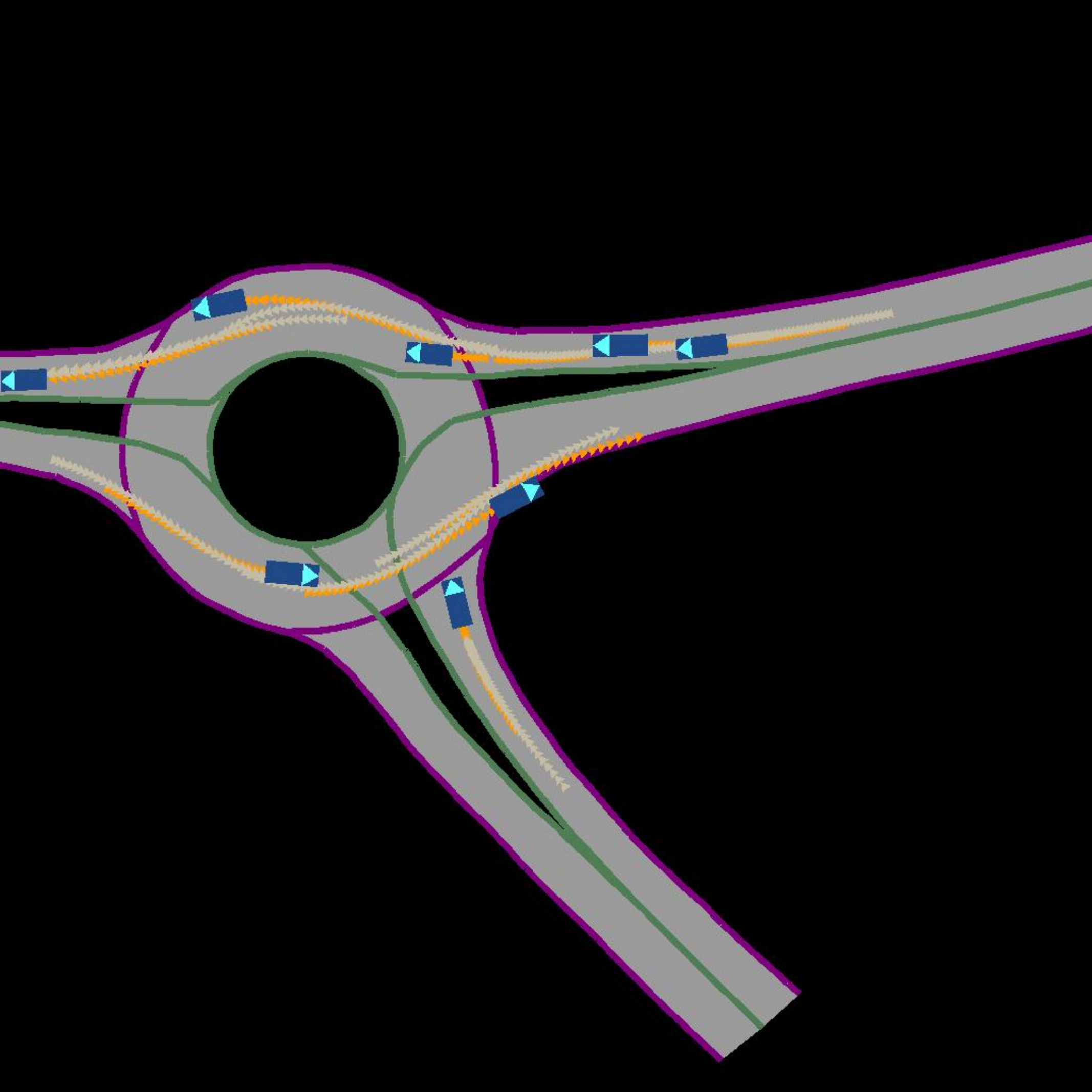}
       \caption*{Standard Diffusion}
        \label{}
    \end{subfigure}
    \hfill
    \begin{subfigure}[b]{0.32\textwidth}
        \centering
        \includegraphics[width=1\textwidth]{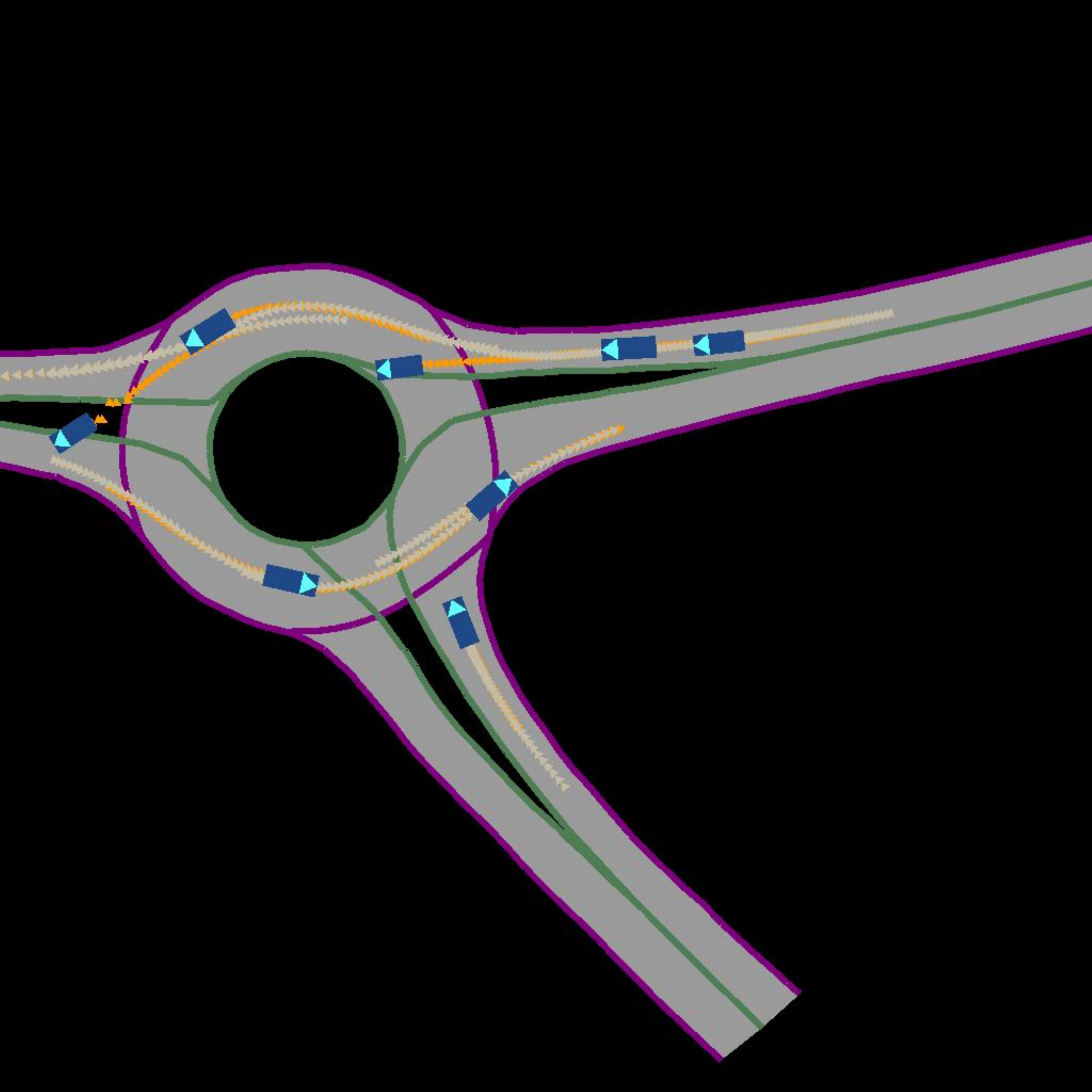}
        \caption*{MPGD w/o projection}
        \label{}
    \end{subfigure}
    \hfill
    \begin{subfigure}[b]{0.32\textwidth}
        \centering
        \includegraphics[width=1\textwidth]{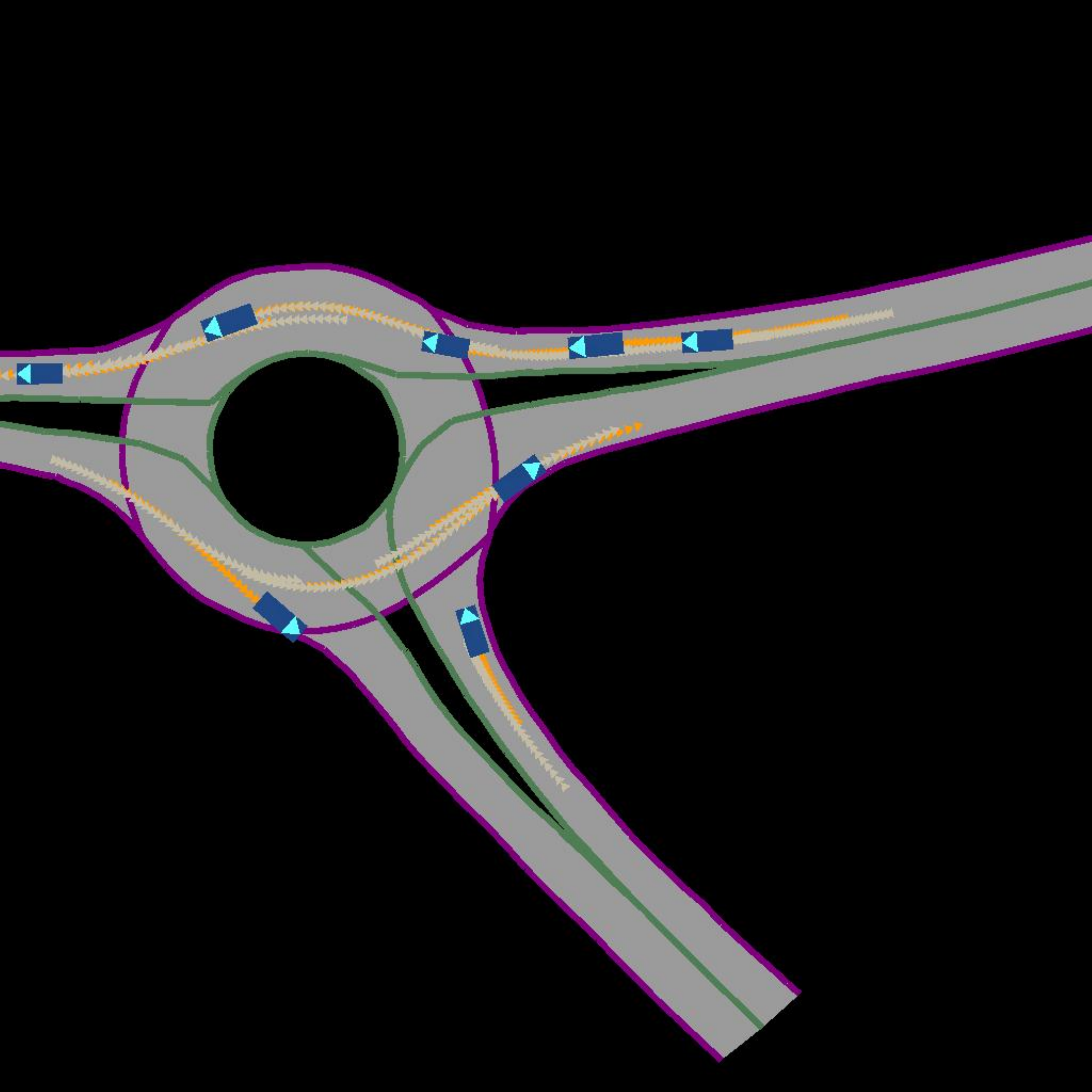}
        \caption*{\method{} (Ours)}
        \label{}
    \end{subfigure}
\end{figure*}

\begin{figure*}[!th]
    \centering
    \begin{subfigure}[b]{0.32\textwidth}
        \centering
       \includegraphics[width=1\textwidth]{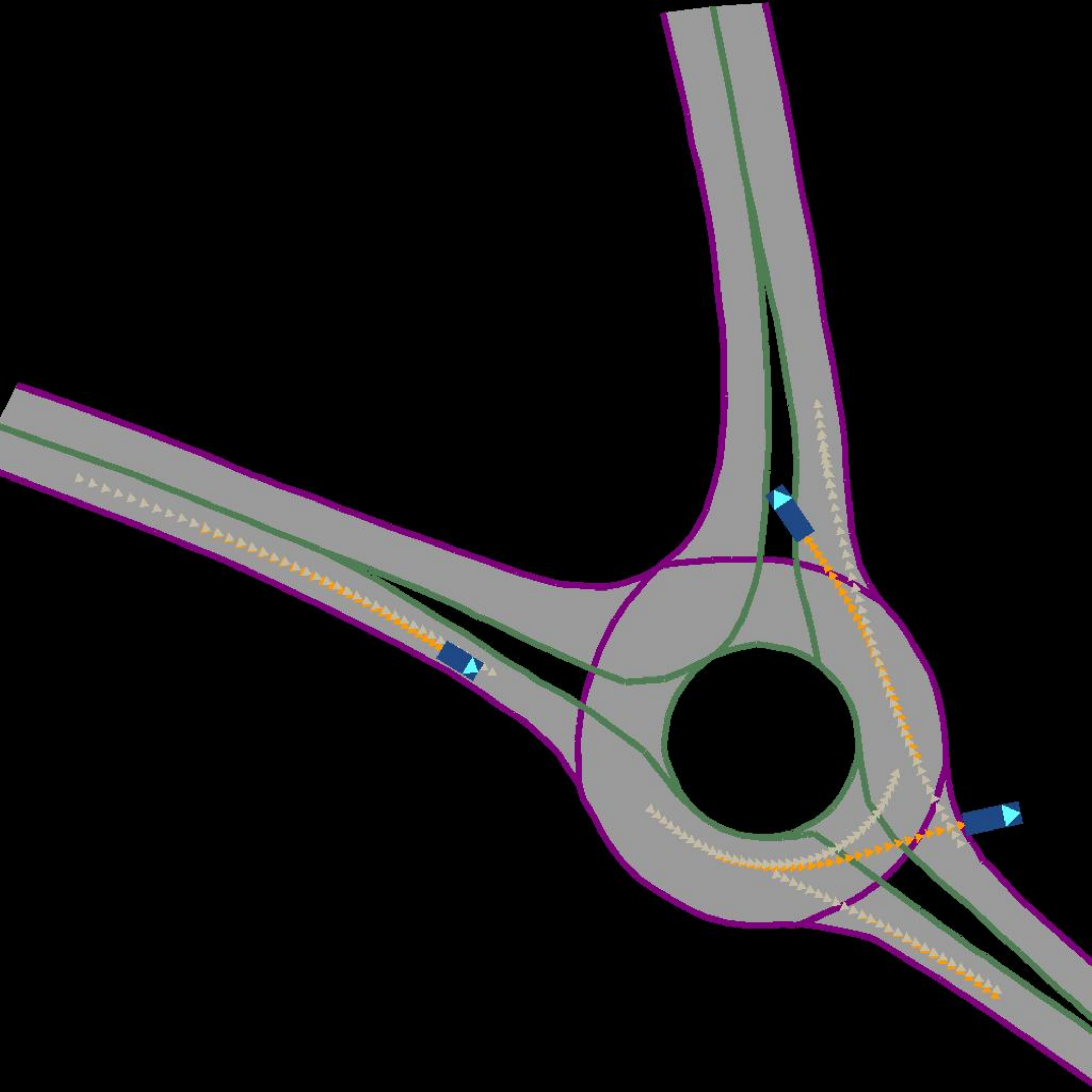}
       \caption*{Standard Diffusion}
        \label{}
    \end{subfigure}
    \hfill
    \begin{subfigure}[b]{0.32\textwidth}
        \centering
        \includegraphics[width=1\textwidth]{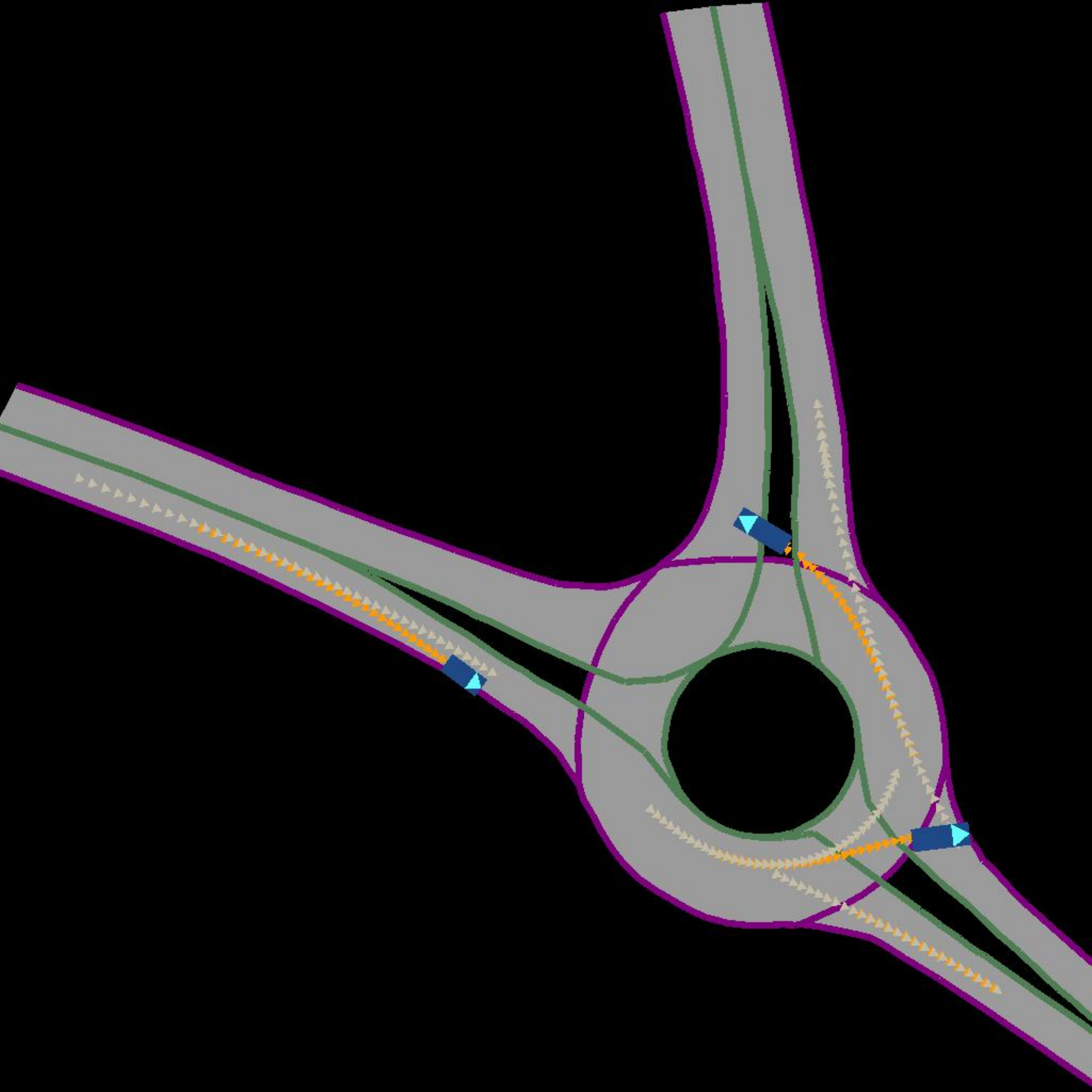}
        \caption*{MPGD w/o projection}
        \label{}
    \end{subfigure}
    \hfill
    \begin{subfigure}[b]{0.32\textwidth}
        \centering
        \includegraphics[width=1\textwidth]{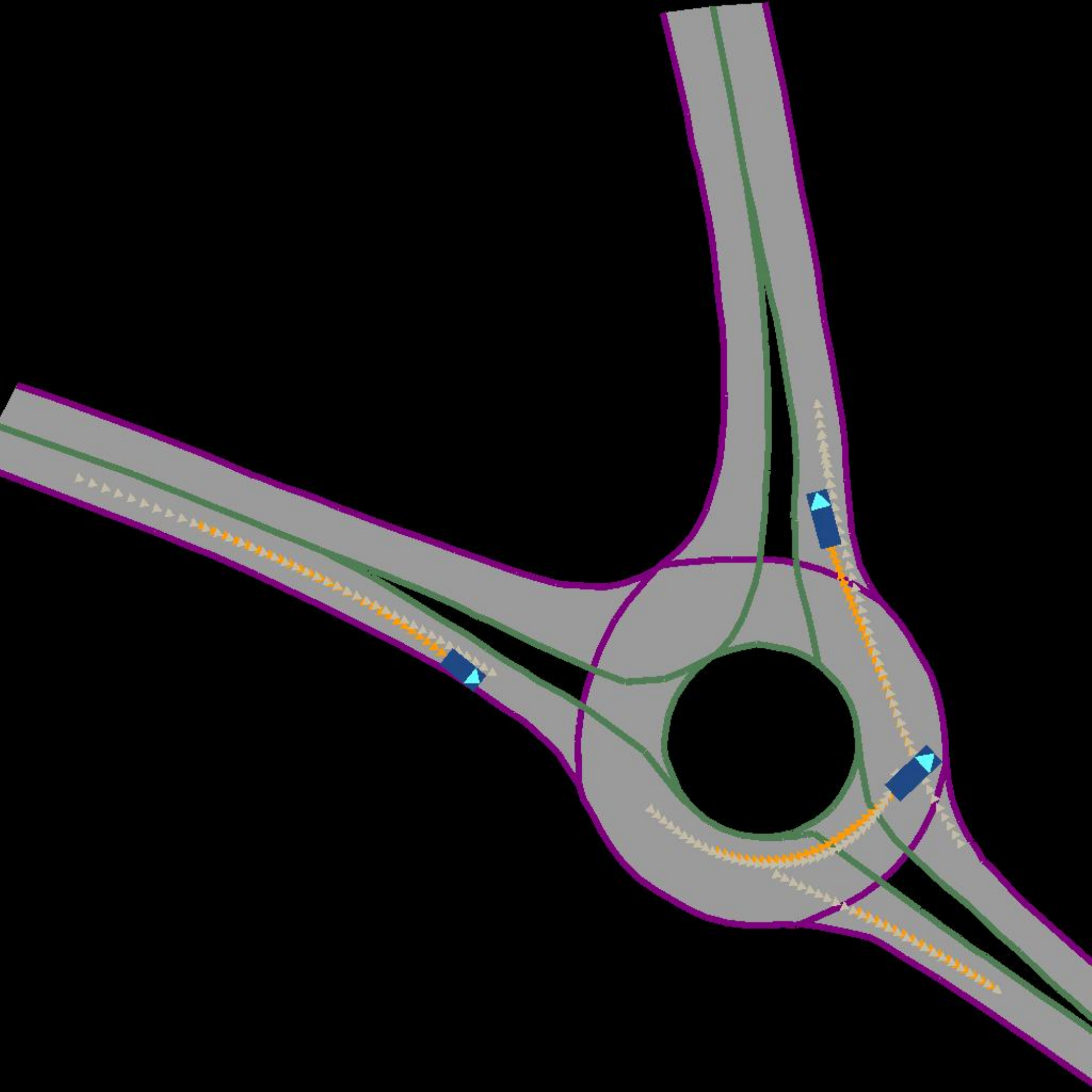}
        \caption*{\method{} (Ours)}
        \label{}
    \end{subfigure}
\end{figure*}

\begin{figure*}[!th]
    \centering
    \begin{subfigure}[b]{0.32\textwidth}
        \centering
       \includegraphics[width=1\textwidth]{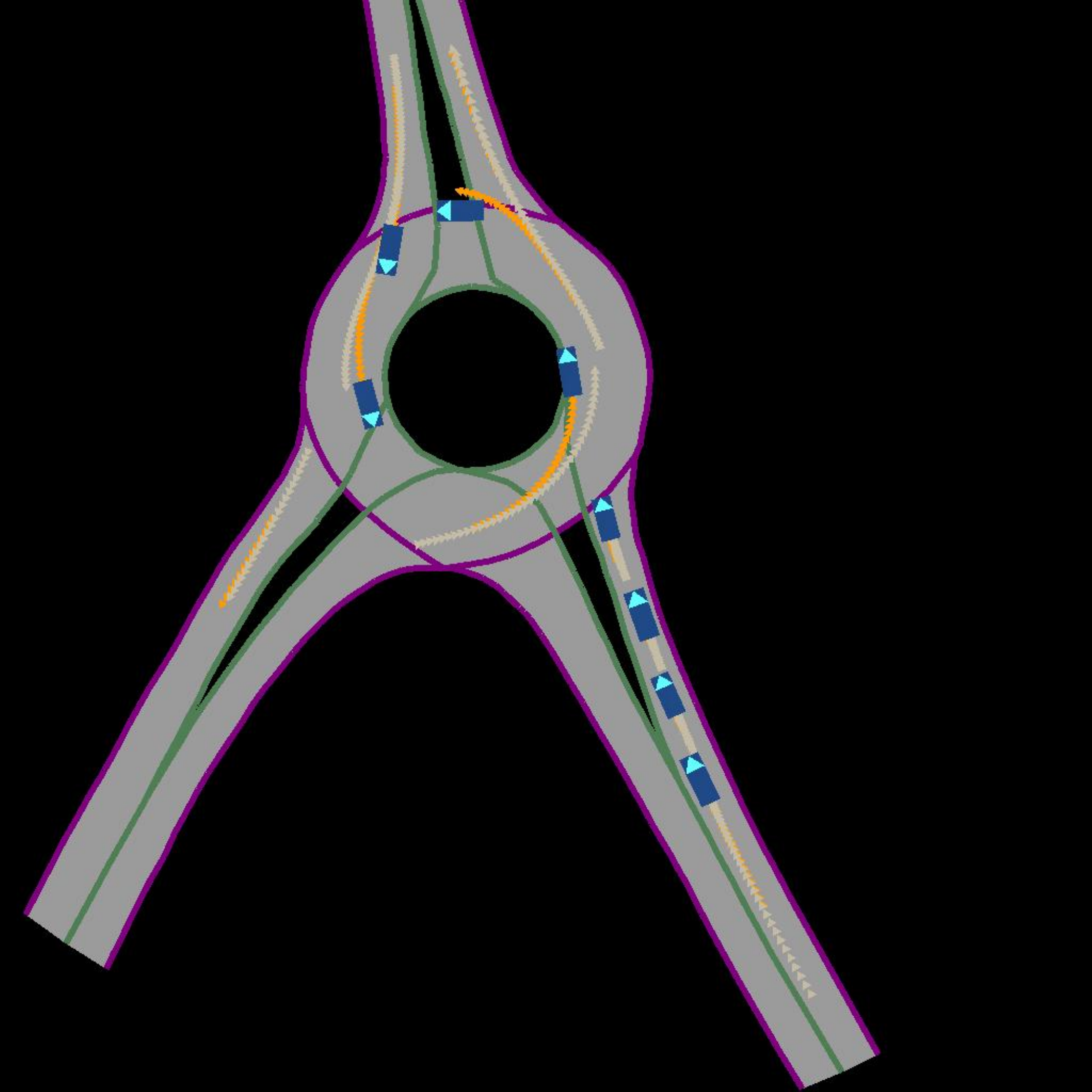}
       \caption*{Standard Diffusion}
        \label{}
    \end{subfigure}
    \hfill
    \begin{subfigure}[b]{0.32\textwidth}
        \centering
        \includegraphics[width=1\textwidth]{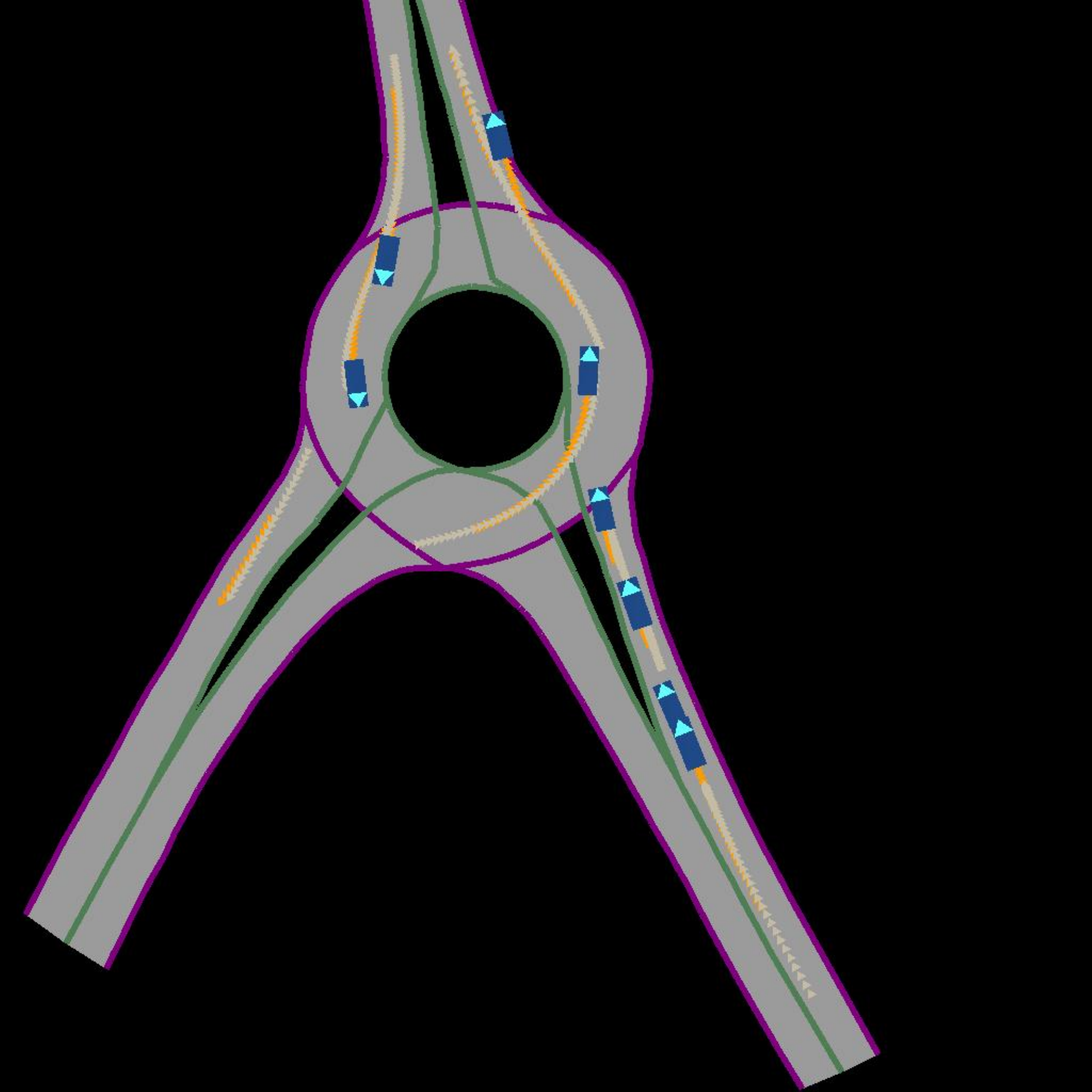}
        \caption*{MPGD w/o projection}
        \label{}
    \end{subfigure}
    \hfill
    \begin{subfigure}[b]{0.32\textwidth}
        \centering
        \includegraphics[width=1\textwidth]{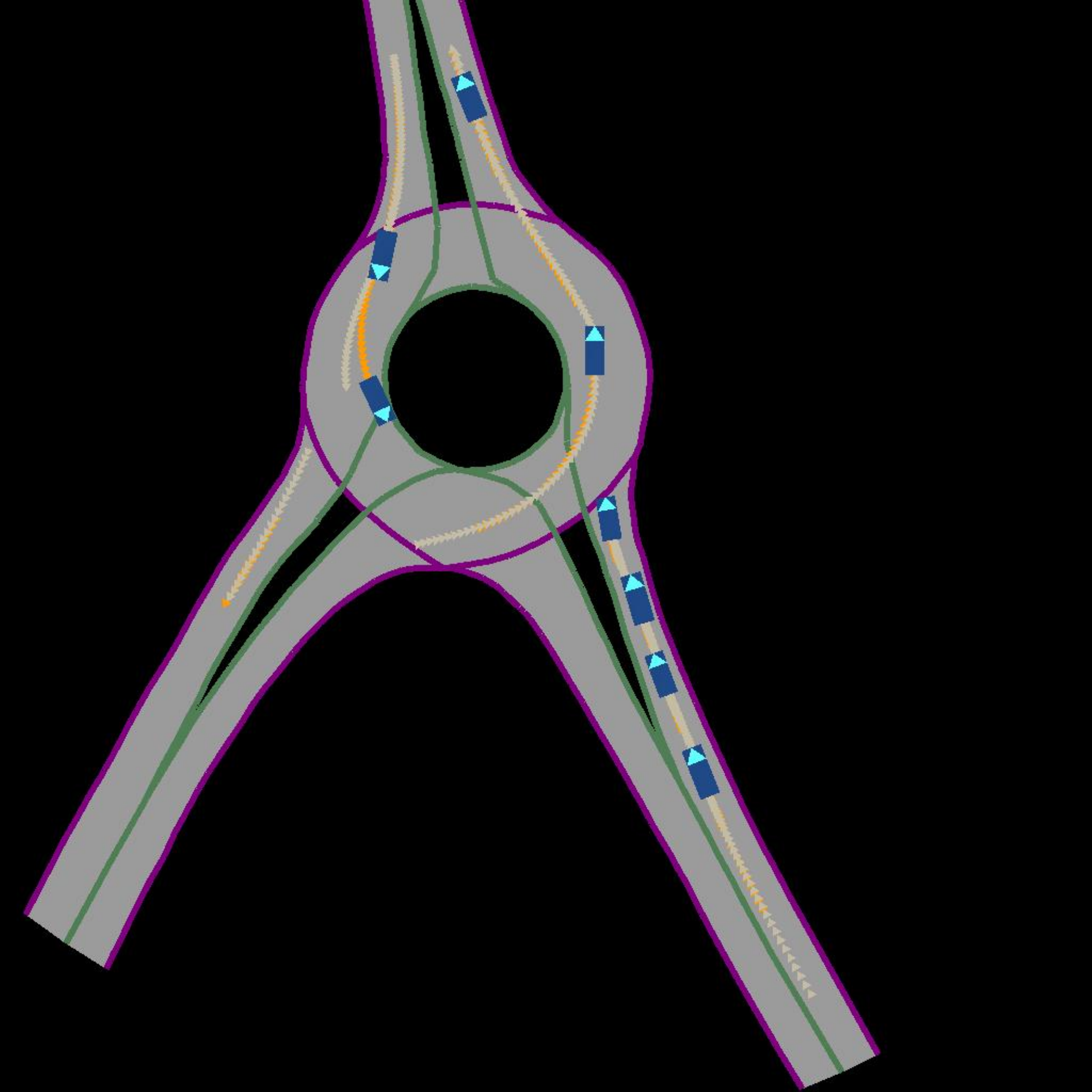}
        \caption*{\method{} (Ours)}
        \label{}
    \end{subfigure}
\end{figure*}

\begin{figure*}[!th]
    \centering
    \begin{subfigure}[b]{0.32\textwidth}
        \centering
       \includegraphics[width=1\textwidth]{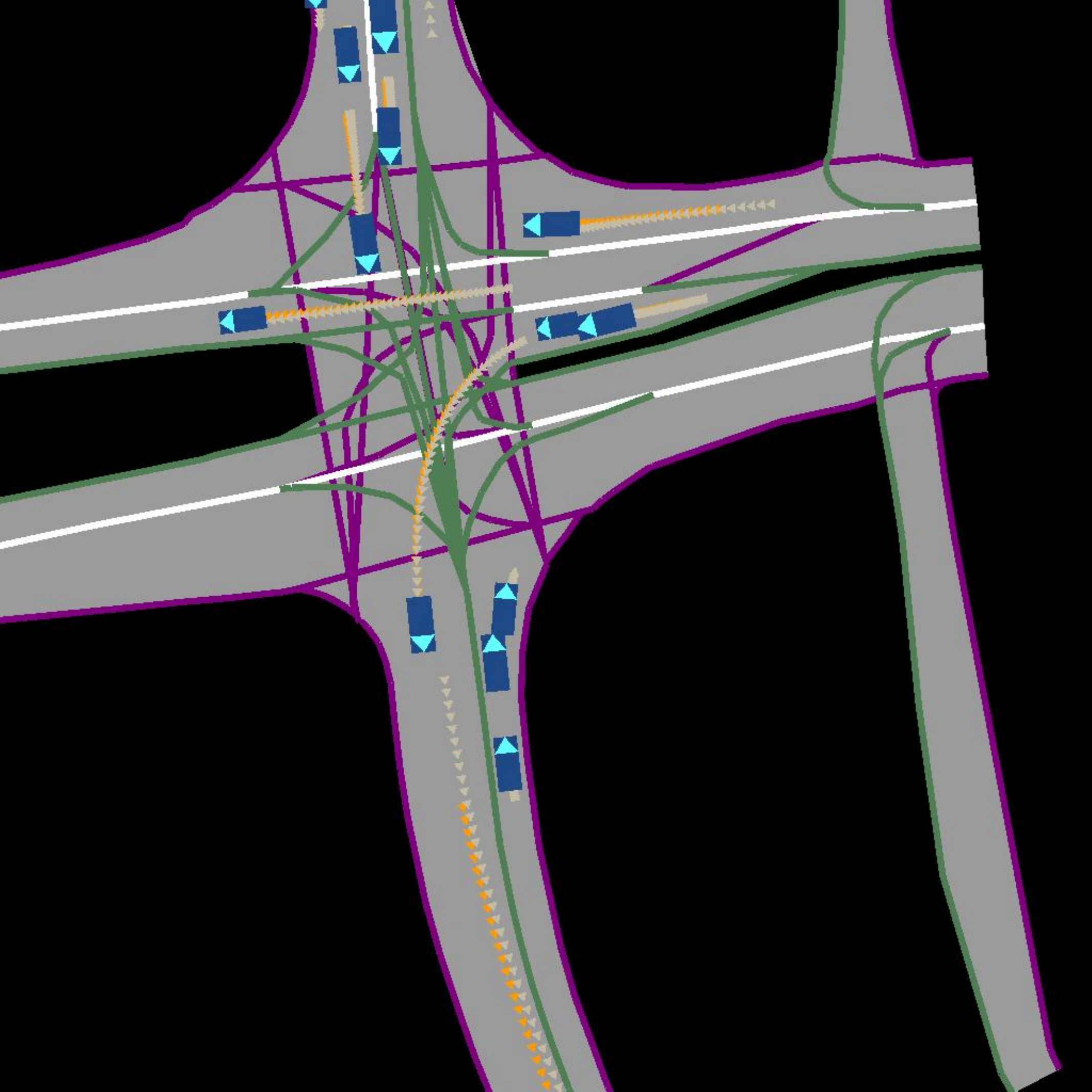}
       \caption*{Standard Diffusion}
        \label{}
    \end{subfigure}
    \hfill
    \begin{subfigure}[b]{0.32\textwidth}
        \centering
        \includegraphics[width=1\textwidth]{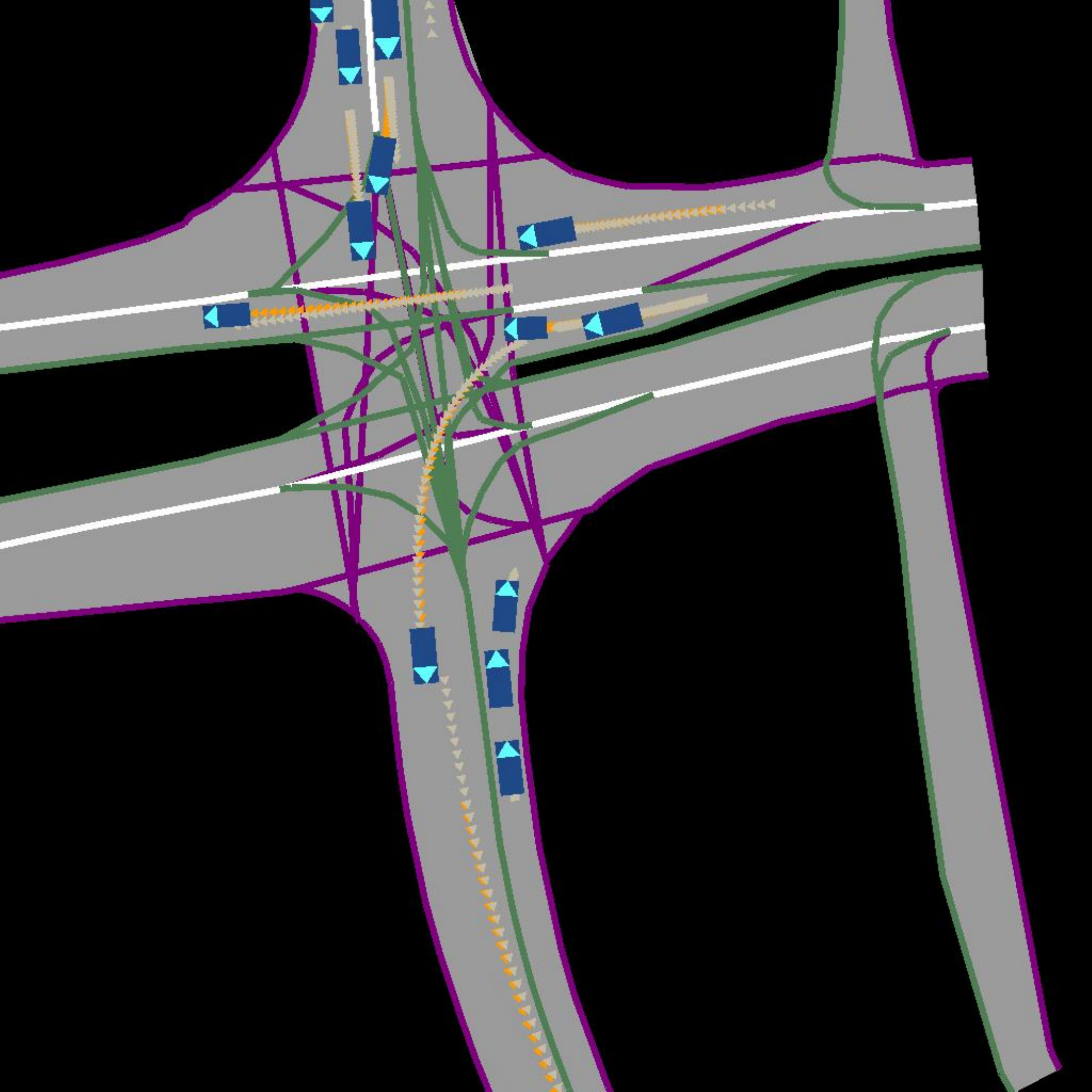}
        \caption*{MPGD w/o projection}
        \label{}
    \end{subfigure}
    \hfill
    \begin{subfigure}[b]{0.32\textwidth}
        \centering
        \includegraphics[width=1\textwidth]{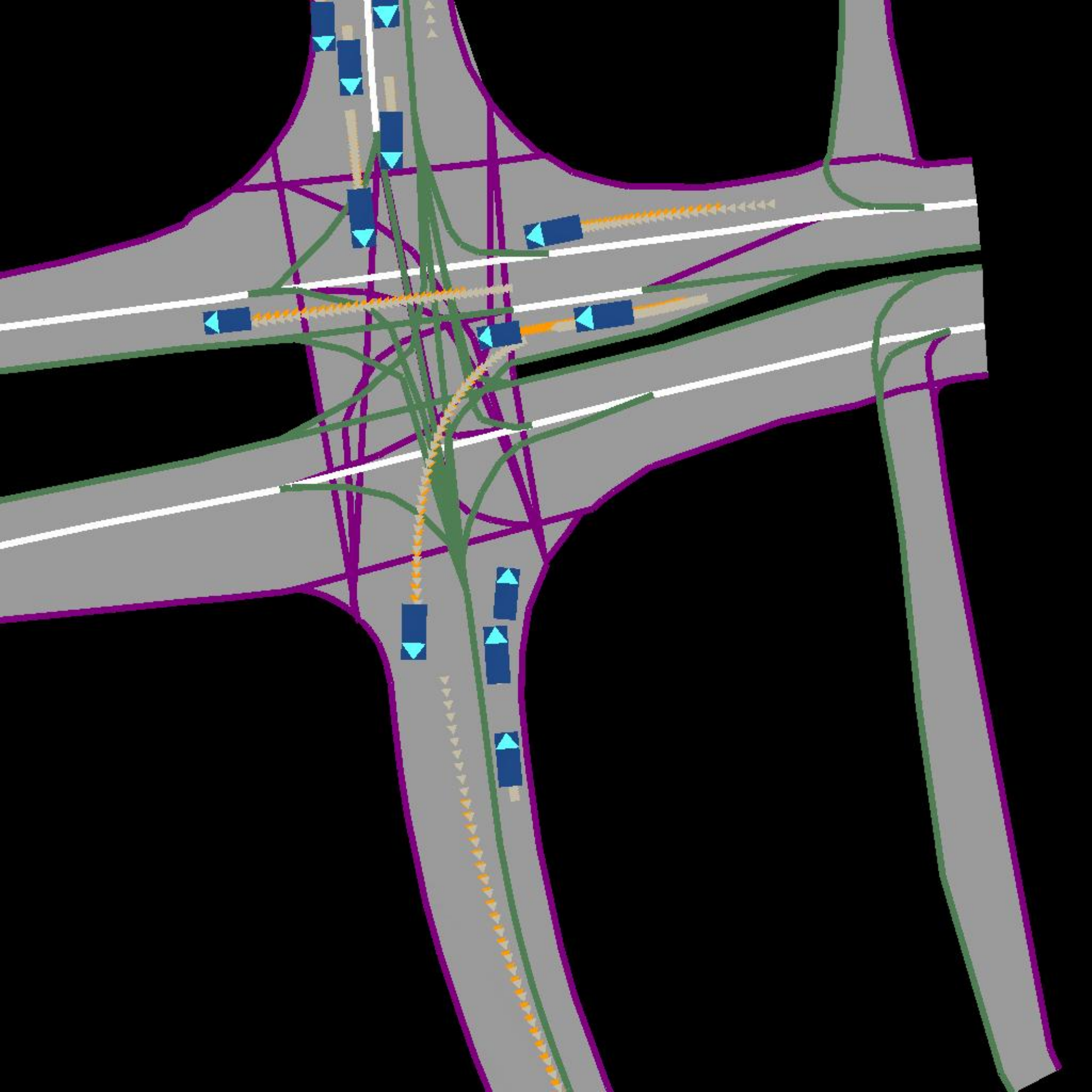}
        \caption*{\method{} (Ours)}
        \label{}
    \end{subfigure}
\end{figure*}

\begin{figure*}[!th]
    \centering
    \begin{subfigure}[b]{0.32\textwidth}
        \centering
       \includegraphics[width=1\textwidth]{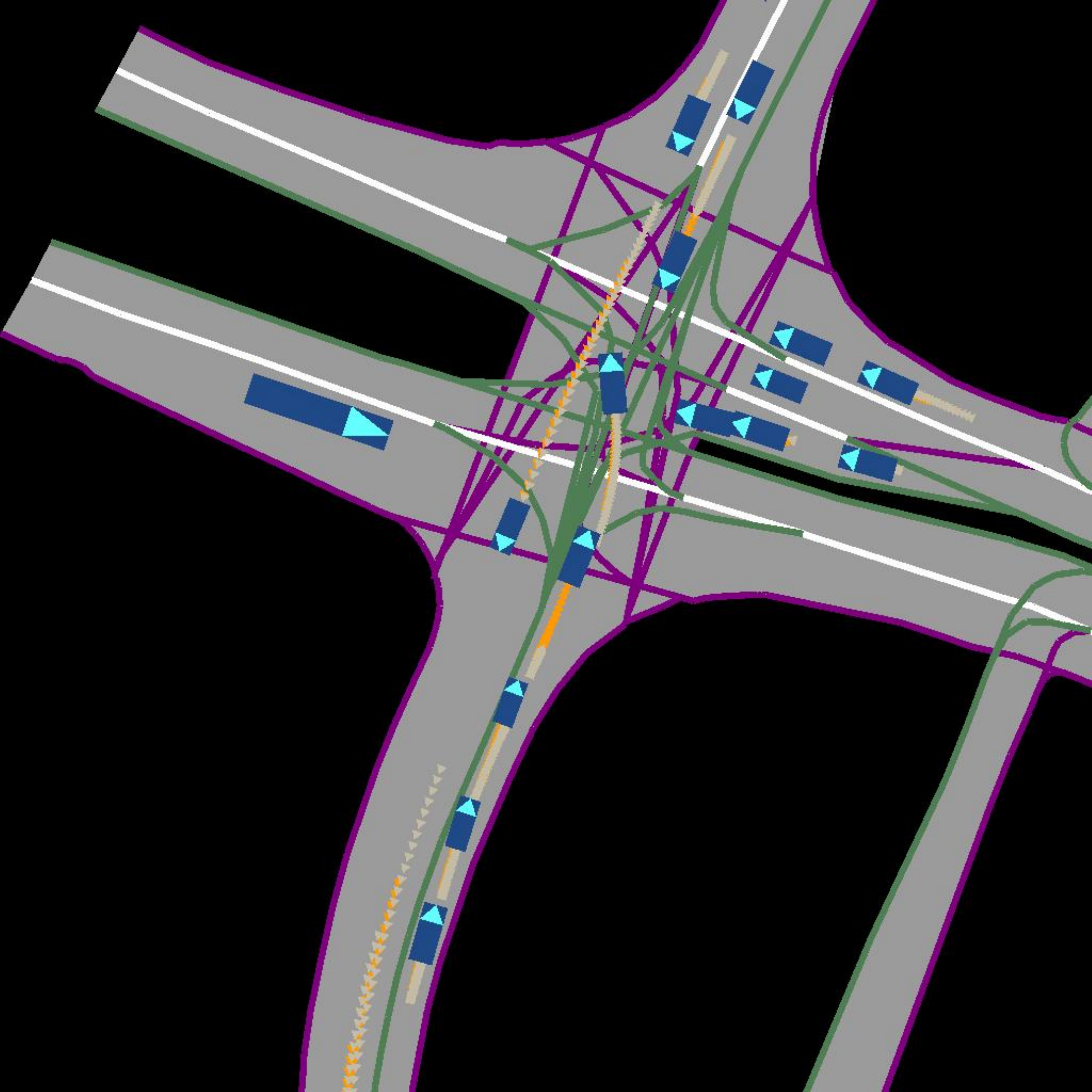}
       \caption*{Standard Diffusion}
        \label{}
    \end{subfigure}
    \hfill
    \begin{subfigure}[b]{0.32\textwidth}
        \centering
        \includegraphics[width=1\textwidth]{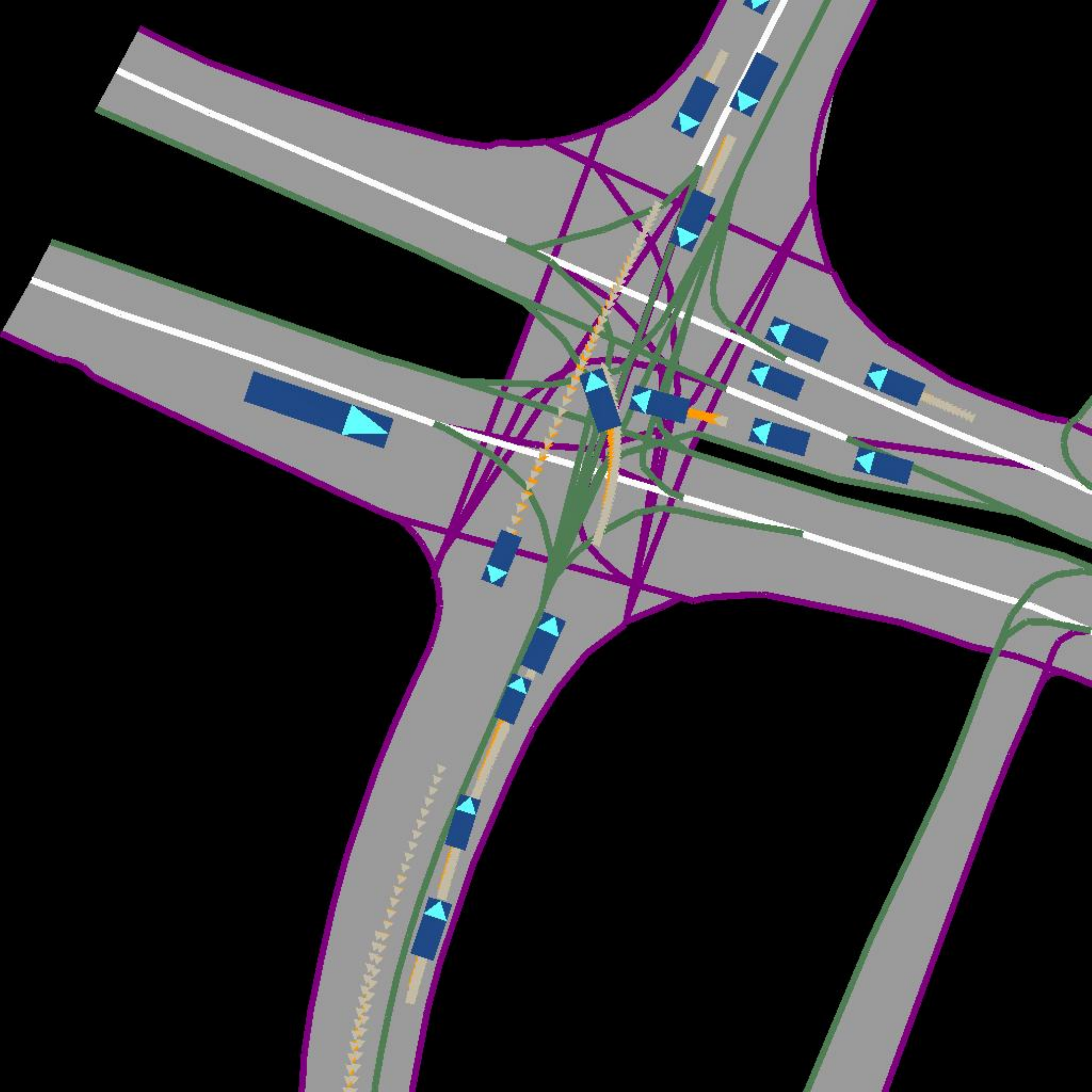}
        \caption*{MPGD w/o projection}
        \label{}
    \end{subfigure}
    \hfill
    \begin{subfigure}[b]{0.32\textwidth}
        \centering
        \includegraphics[width=1\textwidth]{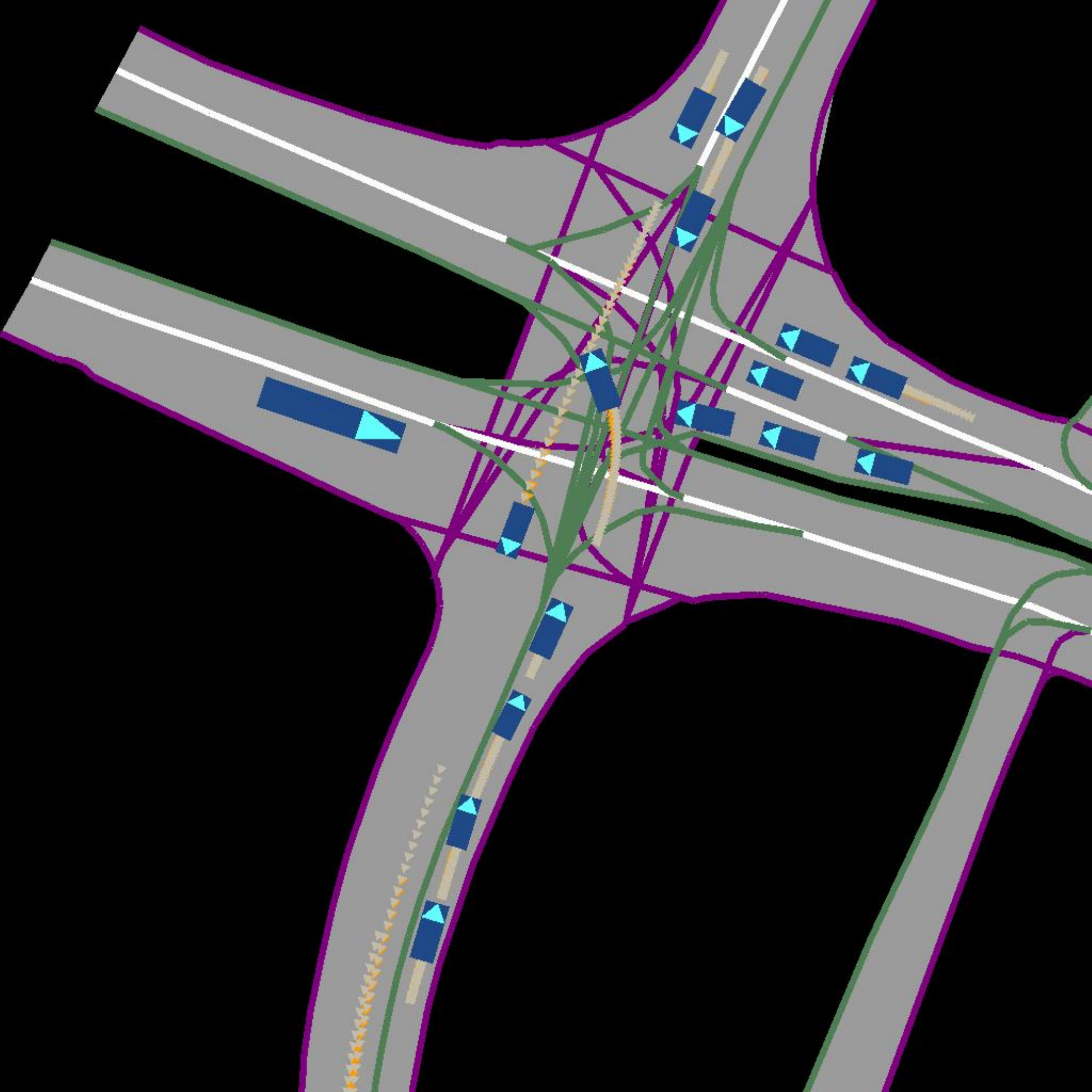}
        \caption*{\method{} (Ours)}
        \label{}
    \end{subfigure}
\end{figure*}

\end{document}